%% file: paper.tex
\documentclass[twoside,leqno,twocolumn]{article}  

\newif\ifapx
\apxtrue 

\usepackage{times}
\usepackage[english]{babel}

\usepackage{ltexpprt}
\usepackage{algorithmic}
\usepackage{algorithm}
\usepackage{amsmath}
\usepackage{amssymb}
\usepackage{verbatim}
\usepackage{multirow}
\usepackage{booktabs}
\usepackage{url}
\usepackage{cite}
\usepackage{xspace}
\usepackage{array}
\usepackage{dcolumn}
\usepackage{graphicx}
\usepackage[tight,center]{subfigure}
\usepackage{color}
\usepackage{makeidx}
\usepackage{verbatim}
\usepackage{caption}
\usepackage{float}
\usepackage{enumitem}
\usepackage{pifont}
\usepackage{enumitem} 
\usepackage{calc} 
\usepackage{cancel}

\usepackage{tikz}
\usetikzlibrary{snakes,arrows,shapes,positioning,fit,calc,patterns}
\usepackage{pgfplots}

\newcommand{\otoprule }{\midrule[\heavyrulewidth]}

\newcommand{\ourmaintitle}{Universal Dependency Analysis}
\newcommand{\ourtitle}{\ourmaintitle}
\newcommand{\oururl}{\url{http://eda.mmci.uni-saarland.de/uds/}}

\newcommand{\codeurl}{\oururl}

\include{defines}

\begin{document}

\title{\ourtitle}

\author{
Hoang-Vu Nguyen\thanks{Max Planck Institute for Informatics and Saarland University, Germany. Email: \texttt{\{hnguyen,pmandros,jilles\}@mpi-inf.mpg.de}} \hspace{2.0cm}
Panagiotis Mandros\footnotemark[1] \hspace{2.0cm}
Jilles Vreeken\footnotemark[1]
}

\date{}

\maketitle

\begin{abstract}
\small\baselineskip=9pt%
Most data is multi-dimensional. Discovering whether any subset of dimensions, or \emph{subspaces}, of such data is significantly correlated is a core task in data mining. To do so, we require a measure that quantifies how correlated a subspace is. For practical use, such a measure should be \emph{universal} in the sense that it captures correlation in subspaces of any dimensionality and allows to meaningfully compare correlation scores across different subspaces, regardless how many dimensions they have and what specific statistical properties their dimensions possess. Further, it would be nice if the measure can non-parametrically and efficiently capture both linear and non-linear correlations.

In this paper, we propose \uds, a multivariate correlation measure that fulfills all of these desiderata. In short, we define \uds based on cumulative entropy and propose a principled normalization scheme to bring its scores across different subspaces to the same domain, enabling \textit{universal} correlation assessment. \uds is purely non-parametric as we make no assumption on data distributions nor types of correlation. To compute it on empirical data, we introduce an efficient and non-parametric method. Extensive experiments show that \uds outperforms state of the art.

\end{abstract}



\input{01intro}

\input{02pre}

\input{03cmi}

\input{04uds}

\input{05comp}


\input{06rl}

\input{07exp}

\input{08con}

\section*{Acknowledgements}
The authors are supported by the Cluster of Excellence ``Multimodal Computing and Interaction'' within the Excellence Initiative of the German Federal Government.

\bibliographystyle{abbrv}
\bibliography{bib/citation}

\ifapx
\appendix
\include{appendix}

\fi

\end{document}

%% file: defines.tex
\newcommand{\uds}{\textsc{uds}\xspace}
\newcommand{\udsr}{\textsc{uds}$_{\neg r}$\xspace}
\newcommand{\cmi}{\textsc{cmi}\xspace}
\newcommand{\mac}{\textsc{mac}\xspace}
\newcommand{\hics}{\textsc{hics}\xspace}
\newcommand{\qr}{\textsc{qr}\xspace}
\newcommand{\ce}{\textsc{ce}\xspace}

\newcommand{\dbscan}{\textsc{dbscan}\xspace}
\newcommand{\lof}{\textsc{lof}\xspace}

\newcommand{\D}{\mathbf{D}}
\newcommand{\Dsc}{\mathcal{G}}
\newcommand{\Xb}{\mathbf{X}}
\newcommand{\X}{X}
\newcommand{\pr}{\mathit{pr}}
\newcommand{\pref}{\mathit{f}}
\newcommand{\supp}{\mathit{s}}
\newcommand{\val}{\mathit{val}}
\newcommand{\pos}{\mathit{pos}}

\newcommand{\corr}{\mathit{corr}}
\newcommand{\dsc}{\mathit{dsc}}
\newcommand{\size}{m}
\newcommand{\dima}{n}
\newcommand{\diff}{\mathit{div}}
\newcommand{\dom}{\mathit{dom}}
\newcommand{\maxv}{\mathit{max}}
\newcommand{\minv}{\mathit{min}}

\newcommand{\proofApx}{
\begin{proof}
\ifapx 
 We postpone the proof to Appendix~\ref{sec:proofs}.
\else 
 We postpone the proof to the online Appendix.
\fi 
\end{proof}
}

\makeatletter
\renewcommand*{\@fnsymbol}[1]{\ensuremath{\ifcase#1\or   \circ\or \bullet\or *\or \ddagger\or
   \mathsection\or \mathparagraph\or \|\or **\or \dagger\dagger
   \or \ddagger\ddagger \else\@ctrerr\fi}}
\makeatother

\pgfdeclarelayer{background}
\pgfdeclarelayer{foreground}
\pgfsetlayers{background,main,foreground}

\tikzstyle{block} = [rounded corners, draw=blue!70, fill=white, text width=3.3cm, minimum height=4em]
\tikzstyle{bgblock} = [rounded corners, draw=blue!70, thick, fill=blue!10, text width=3.3cm, minimum height=4em]
\tikzstyle{line} = [draw, -latex', thick,blue!70]

\definecolor{yafaxiscolor}{rgb}{0.3, 0.3, 0.3}

\definecolor{yafcolor1}{rgb}{0.4, 0.165, 0.553}
\definecolor{yafcolor2}{rgb}{0.949, 0.482, 0.216}
\definecolor{yafcolor3}{rgb}{0.47, 0.549, 0.306}
\definecolor{yafcolor4}{rgb}{0.925, 0.165, 0.224}
\definecolor{yafcolor5}{rgb}{0.141, 0.345, 0.643}
\definecolor{yafcolor6}{rgb}{0.965, 0.933, 0.267}
\definecolor{yafcolor7}{rgb}{0.627, 0.118, 0.165}
\definecolor{yafcolor8}{rgb}{0.878, 0.475, 0.686}

\newlength{\yafaxispad}
\newlength{\yaftlpad}
\newlength{\yaflabelpad}
\newlength{\yafaxiswidth}
\newlength{\yafticklen}
\makeatletter
\def\pgfplots@drawtickgridlines@INSTALLCLIP@onorientedsurf#1{}
\makeatother

\newcommand{\yafdrawaxes}[4]{
	\pgfplotstransformcoordinatex{#1}\let\xmincoord=\pgfmathresult 
	\pgfplotstransformcoordinatex{#2}\let\xmaxcoord=\pgfmathresult 
	\pgfplotstransformcoordinatey{#3}\let\ymincoord=\pgfmathresult 
	\pgfplotstransformcoordinatey{#4}\let\ymaxcoord=\pgfmathresult 
	\pgfsetlinewidth{\yafaxiswidth} 
	\pgfsetcolor{yafaxiscolor}
	\pgfpathmoveto{\pgfpointadd{\pgfpointadd{\pgfplotspointrelaxisxy{0}{0}}{\pgfqpointxy{\xmincoord}{0}}}{\pgfqpoint{-0.5\yafaxiswidth}{\yafaxispad}}}
	\pgfpathlineto{\pgfpointadd{\pgfpointadd{\pgfplotspointrelaxisxy{0}{0}}{\pgfqpointxy{\xmaxcoord}{0}}}{\pgfqpoint{0.5\yafaxiswidth}{\yafaxispad}}}
	\pgfpathmoveto{\pgfpointadd{\pgfpointadd{\pgfplotspointrelaxisxy{0}{0}}{\pgfqpointxy{0}{\ymincoord}}}{\pgfqpoint{\yafaxispad}{-0.5\yafaxiswidth}}}
	\pgfpathlineto{\pgfpointadd{\pgfpointadd{\pgfplotspointrelaxisxy{0}{0}}{\pgfqpointxy{0}{\ymaxcoord}}}{\pgfqpoint{\yafaxispad}{0.5\yafaxiswidth}}}
	\pgfusepath{stroke}
}

\newcommand{\yafdrawYaxis}[2]{
	\pgfplotstransformcoordinatey{#1}\let\ymincoord=\pgfmathresult 
	\pgfplotstransformcoordinatey{#2}\let\ymaxcoord=\pgfmathresult 
	\pgfsetlinewidth{\yafaxiswidth} 
	\pgfsetcolor{yafaxiscolor}
	\pgfpathmoveto{\pgfpointadd{\pgfpointadd{\pgfplotspointrelaxisxy{0}{0}}{\pgfqpointxy{0}{\ymincoord}}}{\pgfqpoint{\yafaxispad}{-0.5\yafaxiswidth}}}
	\pgfpathlineto{\pgfpointadd{\pgfpointadd{\pgfplotspointrelaxisxy{0}{0}}{\pgfqpointxy{0}{\ymaxcoord}}}{\pgfqpoint{\yafaxispad}{0.5\yafaxiswidth}}}
	\pgfusepath{stroke}
}

\newcommand{\yafdrawaxisLimits}[4]{
	\pgfplotstransformcoordinatex{#1}\let\xmincoord=\pgfmathresult
	\pgfplotstransformcoordinatex{#2}\let\xmaxcoord=\pgfmathresult
	\pgfplotstransformcoordinatey{#3}\let\ymincoord=\pgfmathresult 
	\pgfplotstransformcoordinatey{#4}\let\ymaxcoord=\pgfmathresult 
	\pgfsetlinewidth{\yafaxiswidth} 
	\pgfsetcolor{yafaxiscolor}
	\pgfpathmoveto{
		\pgfpointadd{
			\pgfpointadd{
				\pgfplotspointrelaxisxy{0}{0}}{
				\pgfqpointxy{\xmincoord}{0}
			}
		}{
			\pgfqpoint{
				-0.5\yafaxiswidth}{
				\yafaxispad
			}
		}
	}
	\pgfpathlineto{
		\pgfpointadd{
			\pgfpointadd{
				\pgfplotspointrelaxisxy{0}{0}
			}{
				\pgfqpointxy{
					\xmaxcoord
				}{0}
			}
		}{
			\pgfqpoint{
				25.5\yafaxiswidth
			}{
				\yafaxispad
			}
		}
	}
	\pgfpathmoveto{\pgfpointadd{\pgfpointadd{\pgfplotspointrelaxisxy{0}{0}}{\pgfqpointxy{0}{\ymincoord}}}{\pgfqpoint{\yafaxispad}{-0.5\yafaxiswidth}}}
	\pgfpathlineto{\pgfpointadd{\pgfpointadd{\pgfplotspointrelaxisxy{0}{0}}{\pgfqpointxy{0}{\ymaxcoord}}}{\pgfqpoint{\yafaxispad}{0.5\yafaxiswidth}}}
	\pgfusepath{stroke}
}

\pgfplotscreateplotcyclelist{yaf}{%
{yafcolor1,mark options={scale=0.75},mark=o}, 
{yafcolor2,mark options={scale=0.75},mark=square},
{yafcolor3,mark options={scale=0.75},mark=triangle},
{yafcolor4,mark options={scale=0.75},mark=o},
{yafcolor5,mark options={scale=0.75},mark=o},
{yafcolor6,mark options={scale=0.75},mark=o},
{yafcolor7,mark options={scale=0.75},mark=o},
{yafcolor8,mark options={scale=0.75},mark=o}} 

\pgfplotscreateplotcyclelist{yaf fill}{%
{fill=yafcolor1,draw=none}, 
{fill=yafcolor2,draw=none},
{fill=yafcolor3,draw=none},
{fill=yafcolor4,draw=none},
{fill=yafcolor5,draw=none},
{fill=yafcolor6,draw=none},
{fill=yafcolor7,draw=none},
{fill=yafcolor8,draw=none}} 

\pgfplotsset{jv ybar/.style={
   ybar, 
   cycle list name=yaf fill,
   xtick = \empty,
   every extra x tick/.style={major tick length=0pt,color=black}, 
   xmajorgrids = false,
}}

\pgfplotsset{jv line/.style={
   no markers,
   cycle list name=yaf,
   log ticks with fixed point,
   y tick label style = {/pgf/number format/set thousands separator = {\,}},
}}
\pgfplotsset{jv line ylog/.style={
   jv line,
}}

\pgfplotsset{axis y line=left, axis x line=bottom,
	tick align=outside,
	tickwidth=\yafticklen,
	clip = false,
    x axis line style= {-, line width = 0pt, color=black!0},
    y axis line style= {-, line width = 0pt, color=black!0},
    x tick style= {line width = \yafaxiswidth, color=yafaxiscolor, yshift = \yafaxispad},
    y tick style= {line width = \yafaxiswidth, color=yafaxiscolor, xshift = \yafaxispad},
    x tick label style = {font=\scriptsize, yshift = \yaftlpad},
    y tick label style = {font=\scriptsize, xshift = \yaftlpad},
    every axis y label/.style = {at = {(ticklabel cs:0.5)}, rotate=90, anchor=center, font=\scriptsize, yshift = -\yaflabelpad},
    every axis x label/.style = {at = {(ticklabel cs:0.5)}, anchor=center, font=\scriptsize, yshift = \yaflabelpad},
    x tick label style = {font=\scriptsize, yshift = 1pt},
    grid = major,
    major grid style  = {dash pattern = on 1pt off 3 pt},
	every axis plot post/.append style= {line width=\yafaxiswidth} ,
	legend cell align = left,
	legend style = {inner sep = 1pt, cells = {font=\scriptsize}},
	legend image code/.code={%
		\draw[mark repeat=2,mark phase=2,#1] 
		plot coordinates { (0cm,0cm) (0.15cm,0cm) (0.3cm,0cm) };%
	} 
}

%% file: 01intro.tex
\section{Introduction} \label{sec:intro}

Correlation analysis is a key element of data mining. It has applications in many domains, including biology and neuroscience~\cite{ReshefEtAl2011,macke:total}. Traditionally, it focuses on two dimensions. More often than not, however, data is multi-dimensional and can contain multivariate correlations e.g.\ hidden in subsets of dimensions, or \textit{subspaces}~\cite{chanda:assoc}. Identifying such subspaces is an important step towards understanding the data. To do so, we need a \emph{universal} correlation measure.

First of all, the measure should be universal in the sense that it is able to detect correlations in subspaces of any dimensionality. Second, for both usability by the domain experts as well as efficient search, it should allow for universal comparison of scores -- regardless the number of dimensions they were computed over, or the statistical properties of these dimensions. Third, for exploratory analysis the measure should be able to non-parametrically capture both linear and non-linear correlations, making no assumption on data distributions nor types of correlation the data may contain. Fourth, the measure should permit non-parametric and efficient computation on empirical data.

Of the above desiderata, although perhaps most important for practical use, universality is most often overlooked in the literature. For instance, commonly used measures in subspace search such as total correlation~\cite{cheng:enclus}, \cmi~\cite{nguyen:cmi}, and quadratic measure of dependency \qr~\cite{nguyen:4s} in general produce scores that are not comparable across subspaces of different dimensionality. Pearson's correlation, \hics~\cite{keller:hics}, and \mac~\cite{nguyen:mac}, while addressing universality, have other issues. In particular, Pearson's correlation is for pairwise linear correlations. \hics relies on high dimensional conditional distributions and hence is prone to the curse of dimensionality. For each subspace, \mac needs to compute correlation scores of all dimension pairs before outputting the final score -- a potential source of inefficiency.

In this paper, we aim at addressing all desiderata. We do so by proposing \uds, for universal dependency score. In short, we define \uds based on cumulative entropy~\cite{rao:cre,crescenzo:ce} -- a new type of entropy permitting non-parametric computation on empirical data. To address universality, we propose a \textit{principled} normalization scheme to bring correlation scores of \uds across different subspaces to the same domain, enabling universal correlation assessment. Further, \uds is highly suited to non-parametric exploratory analysis as we make no assumption on data distributions nor types of correlation. Lastly, we propose a non-parametric method to reliably and efficiently compute \uds on empirical data. Our method does not require to compute pairwise correlations of the involved dimensions and scales near linearly to the data size. Extensive experiments on both synthetic and real-world data sets show that \uds has high statistical power and performs very well in subspace search.

Next, we present the principles of correlation measure. For readability, we put all proofs of the paper in the appendix.

%% file: 02pre.tex
\section{Correlation Measures -- A Brief Primer} \label{sec:pre}

We consider a multivariate data set $\D$ with $\size$ records and $\dima$ real-valued dimensions $X_1, \ldots, X_\dima$. For each dimension $X_i$, we assume that $\dom(X_i) = [\minv_i, \maxv_i]$. Further, we write $p(X_i)$ as its probability distribution function (pdf) and $P(X_i)$ as its cumulative distribution function (cdf).

Each non-empty subset $S \subset \{X_1, \ldots, X_\dima\}$ constitutes a subspace. To discover correlated subspaces, we need to quantify correlation score $\corr(S)$ of $S$. We will mainly use subspace $\{X_1, \ldots, X_d\}$ where $d \in [1, \dima]$ in our analysis. We write $X_{1, \ldots, i}$ for shorthand of $X_1, \ldots, X_i$ ($i \geq 1$).


In principle, $\corr(X_{1, \ldots, d})$ quantifies to how much the relation of $X_{1, \ldots, d}$ deviates from the statistical independence condition, i.e.\ how much their joint probability distribution differs from the product of their marginal probability distributions~\cite{han:info,dcor}. The larger the difference, the higher $\corr(X_{1, \ldots, d})$ is. Formally, we have
\begin{equation} \label{eq:corr}
\textstyle\corr(X_{1, \ldots, d}) = \diff\left(p(X_{1, \ldots, d}), \prod\limits_{i=1}^{d} p(X_i)\right)
\end{equation}
with $\diff$ being an instantiation of a divergence function. An important property for data analysis is that~\cite{renyi:corr} $\corr(X_{1, \ldots, d})$ is non-negative and zero iff $X_{1, \ldots, d}$ are statistically independent, i.e.\ $p(X_{1, \ldots, d}) = \prod\limits_{i=1}^{d} p(X_i)$.

As Eq.~(\ref{eq:corr}) works with multivariate distribution $p(X_{1, \ldots, d})$, when $d$ is large $\corr(X_{1, \ldots, d})$ is prone to the curse of dimensionality. Recognizing this issue, recent work~\cite{nguyen:cmi,drissi:gcre,wata:info} considers factorizing $p(X_{1, \ldots, d})$ and defines
\begin{equation} \label{eq:corrfact}
\textstyle\corr(X_{1, \ldots, d}) = \sum\limits_{i=2}^{d} \diff\left(p(X_i), p(X_i \mid X_{1, \ldots, i-1})\right).
\end{equation}
To uphold the non-negativity and zero score requirements, it suffices that $\diff(p(X_i), p(X_i \mid \cdot))$ must be non-negative and is zero iff $p(X_i) = p(X_i \mid \cdot)$~\cite{nguyen:cmi}.

As we can see, Eq.~(\ref{eq:corrfact}) is a factorized form of Equation~\eqref{eq:corr}. In particular, it computes $\corr(X_{1, \ldots, d})$ by summing up the difference between the marginal distribution $p(X_i)$ and the conditional distribution $p(X_i \mid X_{1, \ldots, i-1})$ for $i \in [2, d]$. In this way, loosely speaking $\corr(S)$ is the sum of the correlation scores of subspaces

$(X_1, X_2), \ldots, (X_1, \ldots, X_i), \ldots, (X_1, \ldots, X_d)$

\noindent if we consider $\diff(p(X_i), p(X_i \mid X_{1, \ldots, i-1}))$ to be the correlation score of the subspace $(X_{1, \ldots, i})$. The advantage of using lower dimensional subspaces is that $\corr(X_{1, \ldots, d})$ in Eq.~(\ref{eq:corrfact}) is more robust to high dimensionality. It, however, in general is variant to the way we form the factorization, i.e.\ the permutation of dimensions used. 

We eliminate such dependence by taking the maximum score over all permutations. By considering the maximum value, we aim at uncovering the best correlation score of the dimensions involved, which is in line with maximal correlation analysis~\cite{breiman:maxcorr,rao:measure}. Formally, letting $\mathcal{F}_d$ be the set of bijective functions $\sigma: \{1, \ldots, d\} \rightarrow \{1, \ldots, d\}$, we have
\begin{equation} \label{eq:corrfactmax}
\textstyle\corr(X_{1, \ldots, d}) = \max\limits_{\sigma \in \mathcal{F}_d} \sum\limits_{i=2}^{d} \diff\left(p(X^\sigma_i), p(X^\sigma_i \mid X^\sigma_{1, \ldots, i-1})\right)
\end{equation}
where $X^\sigma_i = X_{\sigma(i)}$ for $i \in [1, d]$. Like Eq.~(\ref{eq:corrfact}), $\corr(X_{1, \ldots, d})$ in Eq.~(\ref{eq:corrfactmax}) also is more robust to high dimensionality. Further, it is permutation invariant. We design \uds based on this factorized form.

To this end, one important issue however remains open, which is: How to quantify $\diff(p(X_i), p(X_i \mid \cdot))$ to fulfill universality? We address this by means of cumulative entropy~\cite{rao:cre,crescenzo:ce} which we introduce next.

%% file: 03cmi.tex
\section{Cumulative Entropy} \label{sec:cecmi}

In this section, we first provide background of cumulative entropy (\ce). Then, we review Cumulative Mutual Information (\cmi)~\cite{nguyen:cmi} -- a correlation measure that is defined based on \ce but does not address universality.

\subsection{Background of Cumulative Entropy} \label{sec:ce}

In principle, $\ce$ captures the information content (i.e.\ complexity) of a probability distribution. However, different from Shannon entropy, it works with cdfs and can be regarded as a substitute for Shannon entropy on real-valued data.
Formally, the $\ce$ of a real-valued univariate random variable $\X$ is given as
$$\textstyle h(X) = - \int P(x) \log P(x) dx.$$
The conditional $\ce$ of a real-valued univariate random variable $X$ given $Z \in \mathbb{R}^d$ is defined as
$$\textstyle h(X \mid Z) = \int h(X \mid z) p(z) dz.$$
The conditional $\ce$ has two important properties given by the following theorem~\cite{crescenzo:ce,rao:cre}.

\begin{theorem} \label{theo:cenonneg}
$h(X \mid Z) \geq 0$ with equality iff $X$ is a function of $Z$. $h(X \mid Z) \leq h(X)$ with equality iff $X$ is statistically independent of $Z$.
\end{theorem}

Besides, unconditional $\ce$ can be computed in closed-form on empirical data. Let $x_1 \leq \ldots \leq x_{\size}$ be the ordered records of $X$. We have
$\textstyle h(X) = - \sum\limits_{i=1}^{\size-1} (x_{i+1} - x_i) \frac{i}{\size} \log \frac{i}{\size}$.\\
Having introduced \ce, next we review \cmi.

\subsection{Cumulative Mutual Information} \label{sec:cmi}

\cmi follows the factorized model of correlation measures in Eq.~(\ref{eq:corrfactmax}). It instantiates $\diff(p(X_i), p(X_i \mid \cdot))$ by $h(X_i) - h(X_i \mid \cdot)$. Following Theorem~\ref{theo:cenonneg}, this instantiation is non-negative and zero iff $p(X_i) = p(X_i \mid \cdot)$, which is desirable (cf.\ Section~\ref{sec:pre}). Formally, we have:

\begin{Definition} \label{def:cmi} \textbf{Cumulative Mutual Information (\cmi)}\\
The \cmi of $X_{1, \ldots, d}$ is
$$\textstyle\cmi(X_{1, \ldots, d}) = \max\limits_{\sigma \in \mathcal{F}_d} \sum\limits_{i=2}^d h(X_{\sigma(i)}) - h(X_{\sigma(i)} \mid X_{\sigma(1), \ldots, \sigma(i-1)})$$
where $h(X_{\sigma(i)} \mid X_{\sigma(1) \ldots, \sigma(i-1)})$ is $h(X_{\sigma(i)} \mid Z)$ with $Z = X_{\sigma(1), \ldots, \sigma(i-1)}$ being a random vector whose domain is $\dom(X_{\sigma(1)}) \times \cdots \times \dom(X_{\sigma(i-1)})$.
\end{Definition}

Following Theorem~\ref{theo:cenonneg}, \cmi satisfies the non-negativity and zero score requirements of correlation measure. Further, it is non-parametric and can capture different types of correlation. It however has some pitfalls that we explain next.

\subsection{Drawbacks of \cmi} \label{sec:cmidrawback}

First, \cmi does not address universality. In particular, it tends to give higher dimensional subspaces higher scores, as follow.

\begin{lemma} \label{lem:cmibias}
$\cmi(X_{1, \ldots, d}) \leq \cmi(X_{1, \ldots, d + 1})$.
\end{lemma}

\proofApx

In addition, \cmi scores of subspaces with the same dimensionality may also have different scales. To show this, we prove that the scale of $\cmi(X_{1, \ldots, d})$ is dependent on $h(X_1), \ldots, h(X_d)$.

\begin{lemma} \label{lem:bound1}
$\cmi(X_{1, \ldots, d}) \leq \max\limits_{\sigma \in \mathcal{F}_d} \sum\limits_{i=2}^d h(X_{\sigma(i)})$.
\end{lemma}

\proofApx

That is, if two subspaces of the same dimensionality have no common dimension, the upper bound of their \cmi scores may be different, rendering incomparable scales. Combining the results of Lemma~\ref{lem:cmibias} and~\ref{lem:bound1}, we can see that \cmi does not address universality.

As the second issue, to compute \cmi score of two or more dimensions \cmi needs to search for the optimal permutation of the dimensions. The quality of this search is dependent on how well conditional \ce terms are estimated. As \cmi computes such terms using clustering, the search quality and hence the quality of \cmi are dependent on how good the selected clustering algorithm is. Choosing a suitable clustering method, however, is non-trivial.

Our \uds measure is also based on \ce. In contrast to \cmi, it does address all requirements of a good universal correlation measure.

%% file: 04uds.tex
\section{Universal Dependency Analysis} \label{sec:uda}

In short, \uds builds upon and non-trivially extends \cmi to fulfill universality. More specifically, to bring correlation scores to the same scale -- regardless of the number as well as statistical properties of dimensions involved, we first perform normalization of the scores. Second, we judiciously fix permutation of dimensions, i.e.\ no search is required. Third, to avoid data clustering in correlation computation we propose \textit{optimal} discretization to compute conditional \ce terms. Our optimization problem is formulated such that resulting conditional \ce values are not overfit. In the following, we focus on the first two aspects of \uds and postpone the third one to Section~\ref{sec:comp}.

\subsection{Universal Dependency Score Function} \label{sec:uds}

Our goal is to normalize the scores such that they fall into the range $[0, 1]$ where $0$ means no correlation at all. Note that simply normalizing $\cmi(X_{1, \ldots, d})$ by $d$ is not the solution. This is because if we did so, following Lemma~\ref{lem:bound1} the normalized score would be upper-bounded by $\max\limits_{\sigma \in \mathcal{F}_d} \sum_{i=2}^{d} h(X_{\sigma(i)})/d$. As $\max\limits_{\sigma \in \mathcal{F}_d} \sum_{i=2}^{d} h(X_{\sigma(i)})$ is dependent on $X_1, \ldots, X_d$, the requirement of unbiased scores is not met. Hence, we instead perform normalization based on the following observation.

\begin{lemma} \label{lem:bound}
For each permutation $\sigma \in \mathcal{F}_d$,\\
$\textstyle\sum\limits_{i=2}^{d} h(X_{\sigma(i)}) - h(X_{\sigma(i)} \mid X_{\sigma(1), \ldots, \sigma(i-1)}) \leq \sum\limits_{i=2}^d h(X_{\sigma(i)})$\\
with equality iff $X_{\sigma(2), \ldots, \sigma(d)}$ are functions of $X_{\sigma(1)}$.
\end{lemma}

\proofApx

With Lemma~\ref{lem:bound}, we are now ready to define \uds. In particular, we have:

\begin{Definition} \label{def:uds} \textbf{Universal Dependency Score (\uds)}\\
The \uds of $X_{1, \ldots, d}$ is $\textstyle\uds(X_{1, \ldots, d}) = \max\limits_{\sigma \in \mathcal{F}_d} \Phi_\sigma(X_{1, \ldots, d})$\\
where
\begin{equation} \label{eq:phi}
\textstyle\Phi_\sigma(X_{1, \ldots, d}) = \frac{\sum\limits_{i=2}^{d} h(X_{\sigma(i)}) - h(X_{\sigma(i)} \mid X_{\sigma(1), \ldots, \sigma(i-1)})}{\sum\limits_{i=2}^d h(X_{\sigma(i)})}
\end{equation}
and with the convention that $\frac{0}{0} = 0$.
\end{Definition}

That is, $\uds(X_{1, \ldots, d})$ is the maximum \textit{normalized} correlation over all permutation of $X_{1, \ldots, d}$. It is non-parametric and can capture both linear or non-linear correlations. To show that \uds meets universality, we prove some of its relevant properties below.

\begin{lemma} \label{lem:udsbound}
We have\\
$\bullet$ $0 \leq \uds(X_{1, \ldots, d}) \leq 1$.\\
$\bullet$ $\uds(X_{1, \ldots, d}) = 0$ iff $X_{1, \ldots, d}$ are independent.\\
$\bullet$ $\uds(X_{1, \ldots, d}) = 1$ iff there exists $X_i$ such that each $X_j \in \{X_{1, \ldots, d}\} \setminus \{X_i\}$ is a function of $X_i$.
\end{lemma}

\proofApx

Hence, \uds score of \textit{any} subspace falls in the range $[0, 1]$, which means that we can compare correlation scores across different subspaces. Thus, \uds addresses universality. It also meets the non-negativity and zero score requirements. In addition, as the values of \uds are bounded on both sides we can interpret its scores more easily, which is clearly a desirable property for practical correlation analysis~\cite{renyi:corr}.

\subsection{Practical \uds} \label{sec:pracuds}

To compute \uds, we still need to look for the permutation that maximizes the score. There are $d!$ candidate permutations in total. When $d$ is large, the search space is prohibitively large while in general has no clear structure to optimize over. Thus, to boost efficiency we propose a practical (heuristic) version of \uds. It intuitively fixes a permutation for correlation computation and hence saves time. Our experiments confirm that it works very well in practice. Below we provide its formal definition.

\begin{Definition} \label{def:pracuds} \textbf{Practical \uds}\\
The practical \uds of $X_{1, \ldots, d}$ is
$$\uds_{\pr}(X_{1, \ldots, d}) = \Phi_\sigma(X_{1, \ldots, d})$$
where $\sigma \in \mathcal{F}_d$ is such that $h(X_{\sigma(1)}) \geq \ldots \geq h(X_{\sigma(d)})$.
\end{Definition}

\noindent In other words, $\uds_{\pr}$ chooses the permutation corresponding to the sorting of dimensions in descending order of \ce values. We now give the intuition behind this design choice.

We see that to compute \uds, we must find the permutation $\pi \in \mathcal{F}_d$ such that $\Phi_\pi(X_{1, \ldots, d})$ is maximal. To maximize this term, we should minimize its denominator and maximize its numerator; see Eq.~(\ref{eq:phi}). For the former, it would most likely help to exclude $h(X_{\sigma(1)})$, the largest unconditional \ce term. Thus, we expect a permutation where $X_{\sigma(1)}$ appears first to be good.

For the numerator, we make the following observation. Assume that $h(X_i) \geq h(X_j)$, i.e., $X_i$ is more random than $X_j$. Then $h(X_k \mid X_i)$ tends to be smaller than $h(X_k \mid X_j)$~\cite{rao:cre}. For instance, if $X_j$ is deterministic, $h(X_k \mid X_i) \leq h(X_k \mid X_j) = h(X_k)$. So $h(X_k \mid \cdot)$ tends to get further away from $h(X_k)$ as the conditional part becomes more random, and vice versa.

Now assume that $h(X_k) \geq h(X_i)$. If $X_k$ is after $X_i$ in the permutation, $h(X_k)$ will appear in the numerator. However, $h(X_k \mid \cdot)$ tends to get close to $h(X_k)$ as the conditional part containing $X_i$ is less random, i.e.\ $h(X_k) - h(X_k \mid \cdot)$ tends to be small.

If in the permutation $X_k$ instead is before $X_i$, we will have $h(X_i)$ in the numerator. However, $h(X_i \mid \cdot)$ gets further away from $h(X_i)$, i.e., $h(X_i) - h(X_i \mid \cdot)$ tends to be relatively large.

All in all, these suggest that: \textit{dimensions with large \ce values should be placed before those with small \ce values to maximize the numerator}. Our experiments reveal that $\uds_{\pr}$ works very well in practice. In addition, similarly to \uds its scores are in the range $[0, 1]$. Thus, $\uds_{\pr}$ also fulfills universality. Further, it satisfies the non-negativity and zero score requirements. We postpone the formal proof of these to Lemma~\ref{lem:pracudsbound} of Appendix~\ref{sec:proofs}.

We will use $\uds_{\pr}$ in the rest of this paper and simply call it \uds. In summary, \uds allows for universal correlation assessment of subspaces with potentially different dimensionality as well as different statistical properties of the involved dimensions. In addition, its computation does not require searching for an optimal permutation nor computing all pairwise correlations. Next, we explain how to efficiently and non-parametrically compute \uds on empirical data.

%% file: 05comp.tex
\section{Computing \uds} \label{sec:comp}

In this section, for exposition we focus on computing $\uds(X_{1, \ldots, d})$. W.l.o.g., we assume that

$\qquad\textstyle\uds(X_{1, \ldots, d}) = \frac{\sum\limits_{i=2}^{d} h(X_i) - h(X_i \mid X_{1, \ldots, i-1})}{\sum\limits_{i=2}^d h(X_i)}.$

\noindent That is, the permutation of dimensions in $\uds(X_{1, \ldots, d})$, which is fixed, is assumed to be $X_1, \ldots, X_d$.

To compute $\uds(X_{1, \ldots, d})$, we need to compute unconditional \ce terms $h(X_i)$ and conditional \ce terms $h(X_i \mid X_{1, \ldots, i-1})$. Following Section~\ref{sec:ce}, $h(X_i)$ where $i \in [2, d]$ can be computed in closed form.

For the conditional \ce terms, we propose to compute them by \textit{optimal discretization}, which has been shown to preserve correlation structures in data~\cite{nguyen:mac,jilles:causal}. In particular, we formulate the computation of conditional \ce terms as optimization problems where we search for discretizations that \textit{robustly} maximize $\uds(X_{1, \ldots, d})$. Being robust here means that we aim to obtain bins that preserve true correlation in the data, while avoiding seeing structure when there is none, i.e.\ \textit{overfitting}. Note that optimal discretization have been proposed by us in~\cite{nguyen:mac,jilles:causal}. Our current work differ from previous work in that we here introduce a different way to formulate the solution of optimal discretization. Therewith we are able to provide non-trivial details of the algorithmic approach, which consequently gives us more insight on the runtime complexity. More on this will come shortly. For now, we compute $\uds(X_{1, \ldots, d})$ as follows.

\vspace{0.5em}
\noindent\textit{Computing $h(X_2 \mid X_1)$:} The value of $h(X_2 \mid X_1)$ depends on how we estimate the distribution of $X_1$, or in other words, how we discretize the realizations of $X_1$. Here, we propose to search for the discretization of $X_1$ that robustly maximizes $h(X_2) - h(X_2 \mid X_1)$. As $h(X_2)$ is fixed given the realizations of $X_2$, maximizing $h(X_2) - h(X_2 \mid X_1)$ is equivalent to minimizing $h(X_2 \mid X_1)$.

\vspace{0.5em}
\noindent\textit{Computing $h(X_i \mid X_{1, \ldots, i-1})$ for $i \geq 3$:} Ideally, one would simultaneously search for the optimal discretizations of $X_{1, \ldots, i-1}$ that robustly minimize $h(X_i \mid X_{1, \ldots, i-1})$. This however is very computationally expensive. We overcome this by observing that, to this end, we have already discretized $X_{1, \ldots, i-2}$, and the resulting discretizations are for robustly maximizing $\uds(X_{1, \ldots, d})$. Hence, we choose to search for the discretization of $X_{i-1}$ only. By not re-discretizing any dimension already processed, we strongly reduce runtime.

\vspace{0.5em}
\noindent We now prove that the discretization at each step can be searched efficiently by \textit{dynamic programming}. Na\"ively, the optimization of each step can be cast as: Find the discretization of $X$ minimizing $h(X' \mid I, X)$ where $I$ is the set of dimensions we have already discretized, and $X, X' \in \{X_{1, \ldots, d}\} \setminus I$. However, the more random $X$ is, the smaller $h(X' \mid I, X)$~\cite{rao:cre}. In terms of discretization, this means that when $X$ is discretized into more bins, $h(X' \mid I, X)$ tends to be smaller. Or put differently, this na\"ive optimization may prefer solutions with many bins to those with fewer bins, and hence, potentially pick up spurious correlation.

To avoid this issue, we propose to first search for the optimal discretization at each permissible number of bins. Then, using a \textit{regularizer} we perform model selection to identify the discretization that best balances between correlation preservation and robustness. Our regularizer takes into account the number of bins and hence alleviates the overfitting issue. The flow is as follow: optimal discretization first, then model selection.

\subsection{Optimal Discretization}

Our problem can be stated as: \textit{Given an integer $\lambda \in [1, \size]$, find the discretization of $X$ into $\lambda$ bins that minimizes $h(X' \mid I, X)$ where $I$ is the set of dimensions we have already discretized, and $X, X' \in \{X_{1, \ldots, d}\} \setminus I$}.

\vspace{0.5em}
\noindent\textbf{Proof of dynamic programming.}\ To prove that dynamic programming is applicable, we prove that the optimal solution to the above problem exhibits optimal substructure.

Formally, let $\Dsc_\lambda$ be the set of possible discretizations on $X$ that produce exactly $\lambda$ bins. For each $g \in \Dsc_\lambda$, we let $\{b_1^g, \ldots, b_{\lambda}^g\}$ be the set of bins formed by $g$. Each bin $b_i^g = (l_i^g, u_i^g]$ where $l_1^g = \minv(X)$, $u_{\lambda}^g = \maxv(X)$, and $l_i^g = u_{i-1}^g$ for $i \in [2, \lambda]$. There is a one-to-one correspondence between each discretization and the set of bins it forms. Thus, we will use both interchangeably.

Note that the dimensions in $I$ have been discretized. Their discrete space consists of hypercubes; each has $|I|$ sides corresponding to $|I|$ dimensions. Let $C_1, \ldots, C_k$ be the non-empty hypercubes. It holds that $k < \size$.
For each bin $b_i^g$ ($i \in [1, \lambda]$) and hypercube $C_j$ ($j \in [1, k]$), we write $|C_j, b_i^g|$ as the number of data points falling into the $(|I| + 1)$-dimensional hypercube made up by extending $C_j$ with $b_i^g$. Our optimization problem is equivalent to solving
\begin{equation} \label{eq:opt1}
\textstyle\min\limits_{g \in \Dsc_\lambda} \sum\limits_{i=1}^{\lambda} \sum\limits_{j=1}^k \frac{|C_j, b_i^g|}{\size} h(X' \mid C_j, b_i^g).
\end{equation}
Let $\dsc$ be the optimal solution and $\{b_1^{\dsc}, \ldots, b_{\lambda}^{\dsc}\}$ be its bins. For any bin $b$, we write $h(X' \mid I, b)$ as $h(X' \mid I)$ computed using only the points falling into $b$. We have that
\begin{align} \label{eq:optval1}
& \textstyle\sum\limits_{i=1}^{\lambda} \sum\limits_{j=1}^k \frac{|C_j, b_i^{\dsc}|}{\size} h(X' \mid C_j, b_i^{\dsc}) \\
& \textstyle= \frac{|b_{\lambda}^{\dsc}|}{\size} \underbrace{\textstyle\sum\limits_{j=1}^k \frac{|C_j, b_{\lambda}^{\dsc}|}{|b_{\lambda}^{\dsc}|} h(X' \mid C_j, b_{\lambda}^{\dsc})}_{h(X' \mid I, b_{\lambda}^{\dsc})} \nonumber \\
& \textstyle\quad+\frac{\size - b_{\lambda}^{\dsc}}{\size} \sum\limits_{i=1}^{\lambda-1} \sum\limits_{j=1}^k \frac{|C_j, b_i^{\dsc}|}{\size - b_{\lambda}^{\dsc}} h(X' \mid C_j, b_i^{\dsc}) \nonumber
\end{align}
must be minimal among all discretization $g \in \Dsc_\lambda$. We denote the first term on the right hand side of Eq.~(\ref{eq:optval1}) as \textit{Term1} and the second term as \textit{Term2}.

As $\dsc$ is optimal, $\{b_1^{\dsc}, \ldots, b_{\lambda - 1}^{\dsc}\}$ is the optimal way to discretize into $\lambda - 1$ bins the values $X \leq l_{\lambda}^{\dsc}$. In other words, these bins are the solution to Eq.~(\ref{eq:opt1}) w.r.t.\ $\lambda - 1$ and the values $X \leq l_{\lambda}^{\dsc}$. We prove this by contradiction. In particular, we assume that the optimal solution instead is $\{a_1, \ldots, a_{\lambda - 1}\} \neq \{b_1^{\dsc}, \ldots, b_{\lambda - 1}^{\dsc}\}$. Then, it holds that

$\textstyle\sum\limits_{i=1}^{\lambda-1} \sum\limits_{j=1}^k \frac{|C_j, b_i^{\dsc}|}{\size - b_{\lambda}^{\dsc}} h(X' \mid C_j, b_i^{\dsc})$

$\qquad\qquad> \sum\limits_{i=1}^{\lambda-1} \sum\limits_{j=1}^k \frac{|C_j, a_i|}{\size - b_{\lambda}^{\dsc}} h(X' \mid C_j, a_i)$

\noindent Following Eq.~(\ref{eq:optval1}), this means $\{a_1, \ldots, a_t, b_{\lambda}^{\dsc}\}$ is a better way to discretize $X$ minimizing $h(X' \mid I, X)$, which contradicts our assumption on $\dsc$.

Hence, the optimal solution $\dsc$ exhibits optimal substructure. This motivates us to build a \textit{dynamic programming} algorithm to solve our problem.

\vspace{0.5em}
\noindent\textbf{Algorithmic approach.}\ Our method is summarized in Algorithm~\ref{algo:one}. Here, we first form bins $\{a_1, \ldots, a_\beta\}$ (Line~1). We will explain the rationale of this shortly.

Each term $\supp[i] = \sum_{j=1}^i |a_i|$ is the total support of bins $a_1, \ldots, a_i$. We compute such terms from Lines~6 to~8 for later use in Lines~17 and~18. This step takes $O(\beta)$.

Each term $\pref[j][i]$ where $1 \leq j \leq i \leq \beta$ is equal to $h(X' \mid I, \bigcup_{k=j}^{i} a_k)$, i.e.\ $h(X' \mid I)$ computed using data points contained in $\bigcup_{k=j}^{i} a_k$. These terms are analogous to \textit{Term1}. In Lines~9 to~11, we pre-compute them for efficiency purposes. This step in total takes $O(\size \log \size + \size \beta^2)$. Please refer to Appendix~\ref{sec:complex} for the detailed explanation.

Each value $\val[\lambda][i]$ where $\lambda \in [1, \beta]$ and $i \in [\lambda, \beta]$ stands for $h(X' \mid I, X)$ computed by optimally merging (discretizing) initial bins $a_1, \ldots, a_i$ into $\lambda$ bins. $b[\lambda][i]$ contains the resulting bins. Our goal is to efficiently compute $\val[1\ldots\beta][\beta]$ and $b[1\ldots\beta][\beta]$. To do so, from Lines~12 to~14 we first compute $\val[1][1\ldots\beta]$ and $b[1][1\ldots\beta]$. Then from Lines~15 to~22, we incrementally compute relevant elements of arrays $\val$ and $b$, using the recursive relation described in Eq.~(\ref{eq:optval1}). This is standard dynamic programming. Note that in Line~17, term $\frac{\supp[i] - \supp[j]}{\supp[i]} \pref[j+1][i]$ corresponds to \textit{Term1} while term $\frac{\supp[j]}{\supp[i]} \val[\lambda - 1][j]$ corresponds to \textit{Term2}.

The processing from Lines~12 to~22 takes $O(\beta^3)$. As $\beta \ll \size$, our algorithm in total takes $O(\size \log \size + \size \beta^2)$.

\vspace{0.5em}
\noindent\textbf{Remarks.}\ Notice that in our solution, we form initial bins $\{a_1, \ldots, a_\beta\}$ of $X$ where $\beta \ll \size$. Ideally, one would start with $\size$ bins. However, in the extreme case when all realizations of $X$ are distinct, $h(X' \mid I, X)$ will be zero. This is known as the empty space issue~\cite{NonlinearBook}. On the other hand, by pre-partitioning $X$ in to $\beta$ bins, we ensure that there is sufficient data in each bin for a statistically reliable computation. This also helps to boost efficiency. Choosing a suitable value for $\beta$ is a tradeoff between accuracy and efficiency. We empirically study its effect in Appendix~\ref{sec:sensitivity}.

\begin{algorithm}[tb]
\caption{\uds Optimal Discretization}
\label{algo:one}
\begin{algorithmic}[1]

\STATE Create initial bins $\{a_1, \ldots, a_\beta\}$ of $X$

\STATE Create a double array $\supp[1\ldots\beta]$

\STATE Create a double array $\pref[1\ldots\beta][1\ldots\beta]$

\STATE Create a double array $\val[1\ldots\beta][1\ldots\beta]$

\STATE Create an array $b[1\ldots\beta][1\ldots\beta]$ to store bins

\FOR{$i = 1 \rightarrow \beta$}
	\STATE $\supp[i] = \sum_{j=1}^i |a_i|$
\ENDFOR

\FOR{$1 \leq j \leq i \leq \beta$}
	\STATE $\textstyle\pref[j][i] = h(X' \mid I, \bigcup_{k=j}^{i} a_k)$
\ENDFOR

\FOR{$i = 1 \rightarrow \beta$}
	\STATE $b[1][i] = \bigcup_{k=1}^{i} a_k$ and $\val[1][i] = \pref[1][i]$
\ENDFOR

\FOR{$\lambda = 2 \rightarrow \beta$}
	\FOR{$i = \lambda \rightarrow \beta$}
		\STATE $\pos = \arg\min\limits_{j \in [1, i-1]} \Omega(j, i, \lambda)$
		where $\Omega(j, i, \lambda)$\\
		$= \left(\frac{\supp[i] - \supp[j]}{\supp[i]} \pref[j+1][i] + \frac{\supp[j]}{\supp[i]} \val[\lambda-1][j]\right)$
		
		\STATE $\val[\lambda][i] = \Omega(\pos, i, \lambda)$
		
		\STATE Copy all bins in $b[\lambda-1][\pos]$ to $b[\lambda][i]$
		
		\STATE Add $\bigcup_{k=\pos+1}^{i} a_k$ to $b[\lambda][i]$
	\ENDFOR
\ENDFOR

\STATE Return $\val[1\ldots\beta][\beta]$ and $b[1\ldots\beta][\beta]$
\end{algorithmic}
\end{algorithm}

\subsection{Model Selection} \label{sec:model}

We propose a regularization scheme that helps to balance between correlation preservation and robustness. First, we assume that the dimensions of $I$ are respectively discretized into $e_1, \ldots, e_{|I|}$ bins. We pick the best number of bins $\lambda^{*}$ as follow:
\begin{equation} \label{eq:regularized}
\textstyle\lambda^{*} = \arg\min\limits_{\lambda \in [1, \beta]} \frac{h(X' \mid I, X)}{h(X')} + \frac{H(I, X)}{\log \beta + \sum\limits_{i=1}^{|I|} \log e_i}
\end{equation}
where $\lambda$ is the number of bins that $X$ is discretized into and $H(I, X)$ is the joint entropy of dimensions in $I$ and \textit{discretized} $X$.
In short, when $\lambda$ is small, the first term of Eq.~(\ref{eq:regularized}) is large while the second term is small. Conversely, when $\lambda$ is large, the first term is small while the second term is large. The optimal $\lambda^{*}$ yields the best balance between the two terms, i.e.\ the best balance between the cost of the model and the cost of the data given the model. This will help avoiding choosing many bins when there is no real structure.


%% file: 06rl.tex
\section{Related Work} \label{sec:rl}

Correlation analysis traditionally deals with two random variables. For this pairwise setting, prominent measures include Pearson's correlation, Spearman's correlation, Hilbert-Schmidt independence criterion~\cite{gretton:hsic}, distance correlation~\cite{dcor}, mutual information~\cite{cover:06:elements}, and maximal information coefficient~\cite{ReshefEtAl2011}.

To discover multivariate correlations in multi-dimensional data, recently multivariate measures have been proposed. Total correlation~\cite{wata:info,han:info} is defined based on Shannon entropy. It is a generalization of mutual information to the multivariate setting. It however tends to give higher dimensional subspaces larger scores, regardless if correlations in such subspaces are strong~\cite{cover:06:elements}.

Cumulative mutual information (\cmi)~\cite{nguyen:cmi} which uses cumulative entropy is designed specifically for real-valued data. Like total correlation, \cmi is biased towards high dimensional subspaces (see Section~\ref{sec:cmidrawback}).

Quadratic measures of dependency~\cite{rao:measure,seth:measure,nguyen:ipd} permit empirical computation in closed form. They closely follow the correlation model in Eq.~(\ref{eq:corr}) and define their scores using multivariate joint distributions. Thus, they are susceptible to the curse of dimensionality. Further, they lack a formal normalization scheme to address universality.

Recently, Keller et al.~\cite{keller:hics} propose \hics in the related problem setting. To compute correlation of $X_{1, \ldots, d}$, \hics averages over multiple random runs of the form $\diff\left(p(X_i), p(X_i \mid \{X_{1, \ldots, d}\} \setminus \{X_i\})\right)$ where $X_i$ is selected \textit{randomly} in each run. This causes two issues. First, \hics scores are non-deterministic, making subspace search results potentially unpredictable. Second, by using conditional distributions with $d - 1$ conditions, \hics is also prone to high dimensionality issue.

In earlier work, we proposed \mac~\cite{nguyen:mac} -- a normalized form of total correlation. The score is based on Shannon entropy over discretized data which \mac obtains by optimizing w.r.t.\ cumulative entropy. With \uds we stay closer to the source, as we define and optimize our score using just cumulative entropy; we only use Shannon entropy to choose the number of bins over which to report the score. Further, whereas \mac needs to optimize the dimension order, \uds avoids this and scales better.

%% file: 07exp.tex
\section{Experiments} \label{sec:exp}

In this section, we empirically evaluate \uds. In particular, we first study its statistical power on synthetic data sets. Second, as common in subspace search we plug \uds to existing search algorithms~\cite{keller:hics,nguyen:4s} to mine correlated subspaces. We evaluate output subspaces both quantitatively using clustering and outlier detection, as well as qualitatively.

We compare to \cmi~\cite{nguyen:cmi}, \mac~\cite{nguyen:mac}, and \hics~\cite{keller:hics}. As further baseline, we include \udsr, a variant of \uds that does \textit{not} use regularization. For each competitor, we optimize parameter settings according to their respective papers. For \uds and \udsr, the default setting is $\beta = 20$. Like~\cite{ReshefEtAl2011,nguyen:ipd,nguyen:mac}, we form initial bins $\{a_1, \ldots, a_\beta\}$ by applying equal-frequency binning. We implemented \uds in Java, and make our code available for research purposes.\!\footnote{\codeurl} All experiments were performed single-threaded on an Intel(R) Core(TM) i7-4600U CPU with 16GB RAM. We report wall-clock running times.

\subsection{Statistical Power}

To verify the suitability of \uds to correlation analysis, we first perform statistical tests using synthetic data sets. Here, the null hypothesis is that the data dimensions are statistically independent. To determine the cutoff for testing the null hypothesis, we first generate 100 data sets with \textit{no} correlation. Next, we compute their correlation scores and set the cutoff according to the significance level $\alpha = 0.05$. We then generate 100 data sets with correlation. The power of the measure is the proportion of the 100 new data sets whose correlation scores exceed the cutoff.

We generate each data set with correlation as follow. Let $l = \dima / 2$ where $\dima$ is the desired dimensionality. We generate $\Xb_{l \times 1} = \mathbf{A}_{l \times l} \times \mathbf{Z}_{l \times 1}$ where $Z_i \sim \textit{Gaussian}(0, 1)$ and $\mathbf{A}_{l \times l}$ is fixed with $a_{ij}$ initially drawn from $\textit{Uniform}[0, 1]$. Here, $\Xb_{l \times 1}$ and $\mathbf{Z}_{l \times 1}$ are two vectors, each having $l$ dimensions. We let $\{\X_1, \ldots, \X_l\} = \Xb_{l \times 1}$. Next, we generate $\mathbf{W}_{l \times 1} = \mathbf{B}_{l \times l} \times \Xb_{l \times 1}$ where $\mathbf{B}_{l \times l}$ is fixed with $b_{ij}$ initially drawn from $\textit{Uniform}[0, 0.5]$. Then, using a function $f$ we generate $\X_{i + l} = f(W_i) + e_i$ where $i \in [1, l]$, and $e_i \sim \textit{Gaussian}(0, \sigma)$; we control noise by varying $\sigma$. We use four instantiations of $f$:

$f_1(x) = 2x + 1, \qquad\qquad f_2(x) = x^2 - 2x,$

$f_3(x) = \log(|x| + 1), \qquad f_4(x) = \sin(2x).$

\noindent That is, we test with both linear and non-linear correlations.

To study universality, we test with four cases: 1) data sets with as well as without correlation have the same dimensionality $\dima$; 2) those with correlation have dimensionality $\dima$ while those without have dimensionality $\dima + e$ where $e$ is the number of extra dimensions; 3) those with correlation have dimensionality $\dima + e$ while those without have dimensionality $\dima$; 4) those with as well as without correlation have arbitrary dimensionality. We find 2) and 3) to yield very similar results; hence, we report results of 2) only. For brevity, we further postpone the results of 4) to Appendix~\ref{sec:extraresults}.

The results for case 1) are in Figure~\ref{fig:power}. Here, we set $\size = 4000$ and vary $\dima$. The results for case 2) are in Figure~\ref{fig:power_extra}. Here, we fix $\size = 4000$ and $\dima = 20$, and vary $e$.

\begin{figure}[tb]
\centering
\subfigure[Power on $f_1$]
{{\includegraphics[width=0.23\textwidth]{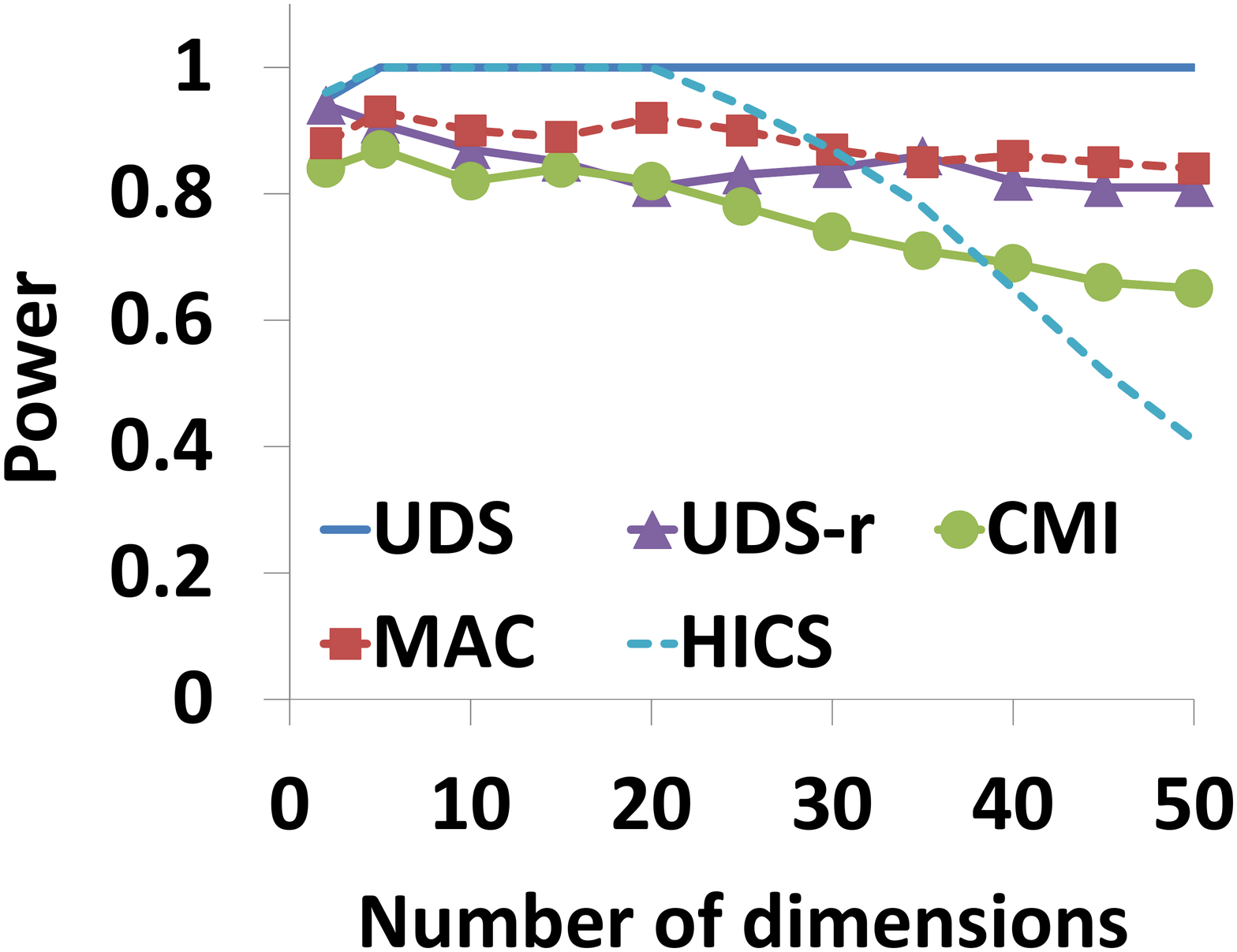}}\label{fig:power_f1}}~
\subfigure[Power on $f_2$]
{{\includegraphics[width=0.23\textwidth]{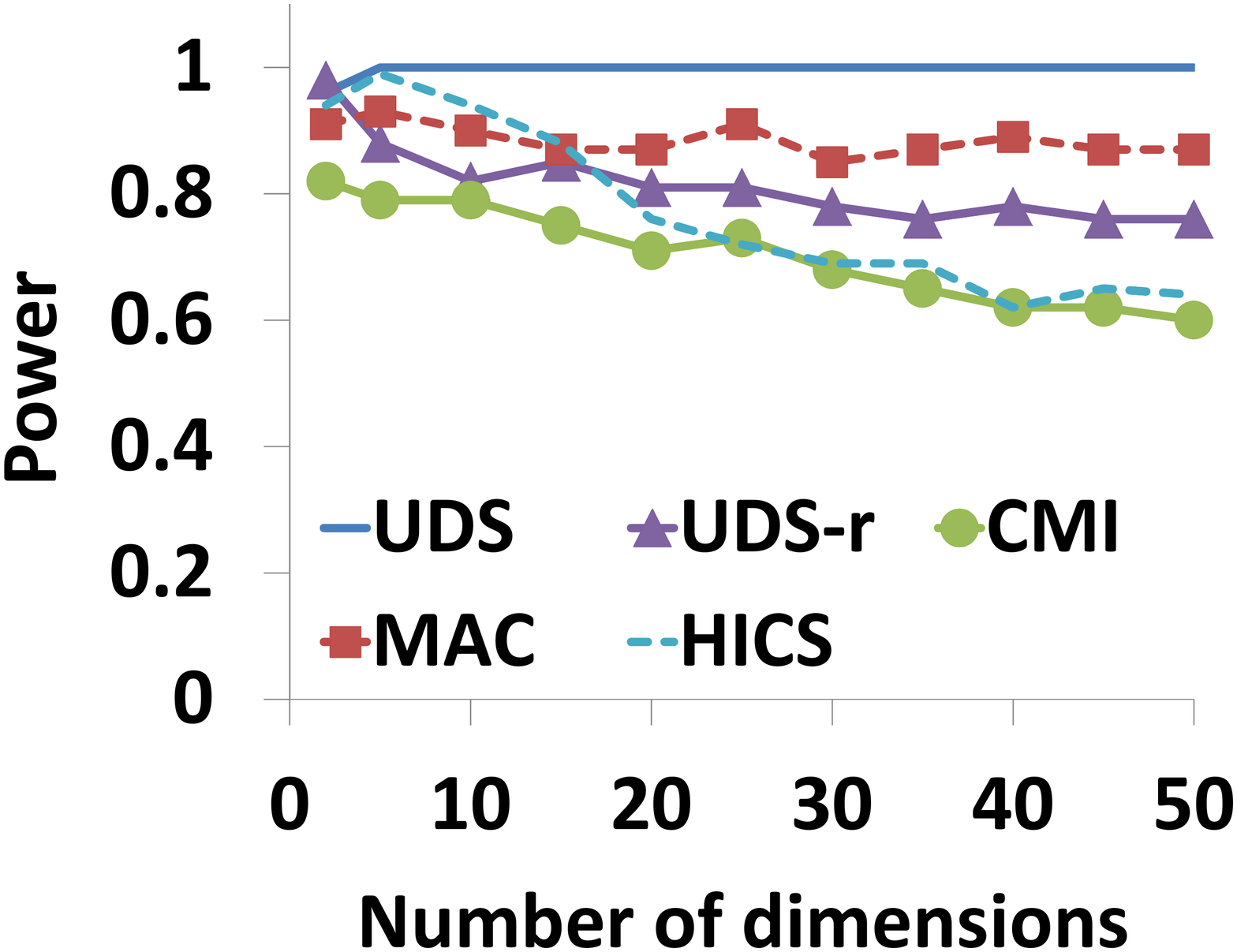}}\label{fig:power_f2}}
\subfigure[Power on $f_3$]
{{\includegraphics[width=0.23\textwidth]{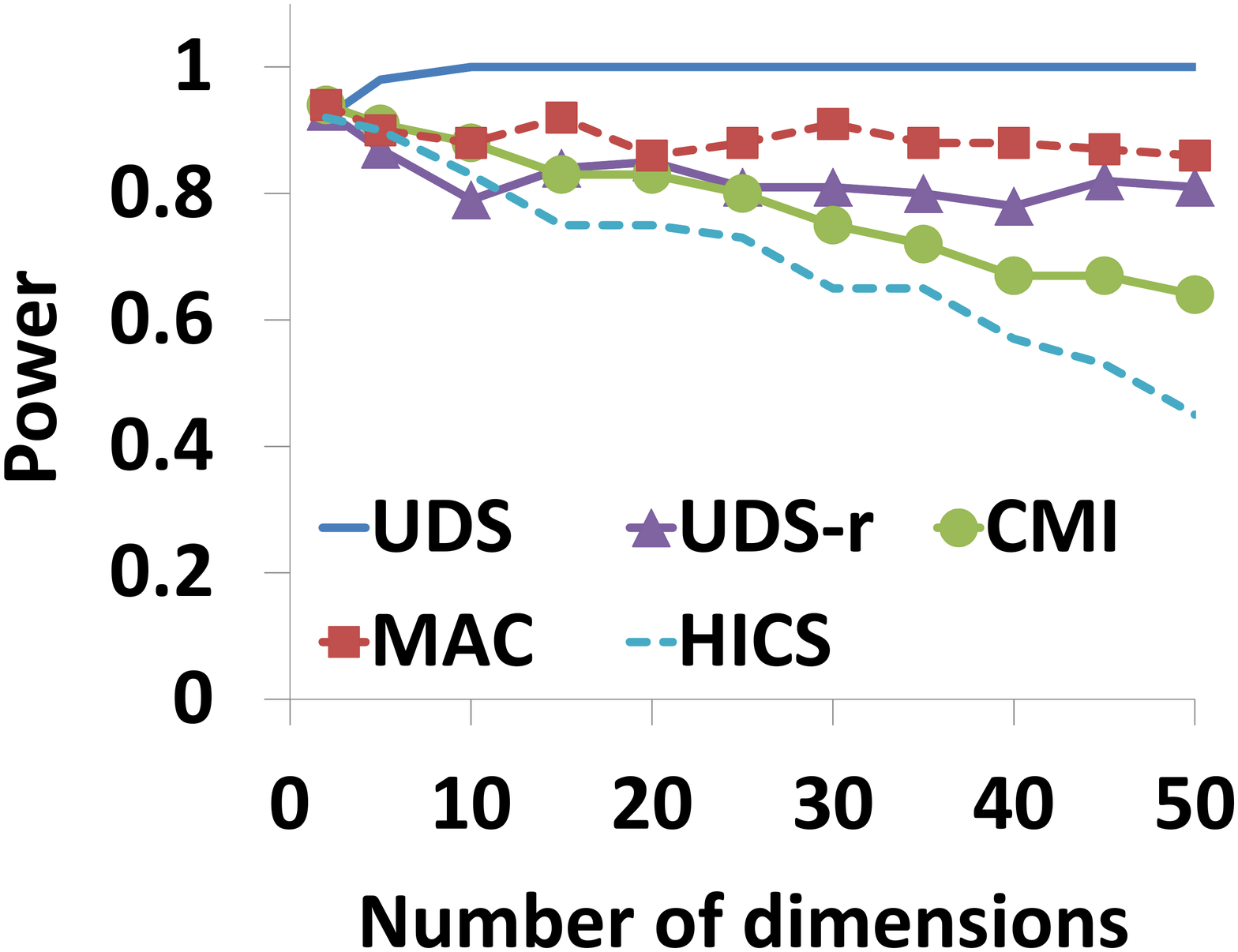}}\label{fig:power_f3}}~
\subfigure[Power on $f_3$]
{{\includegraphics[width=0.23\textwidth]{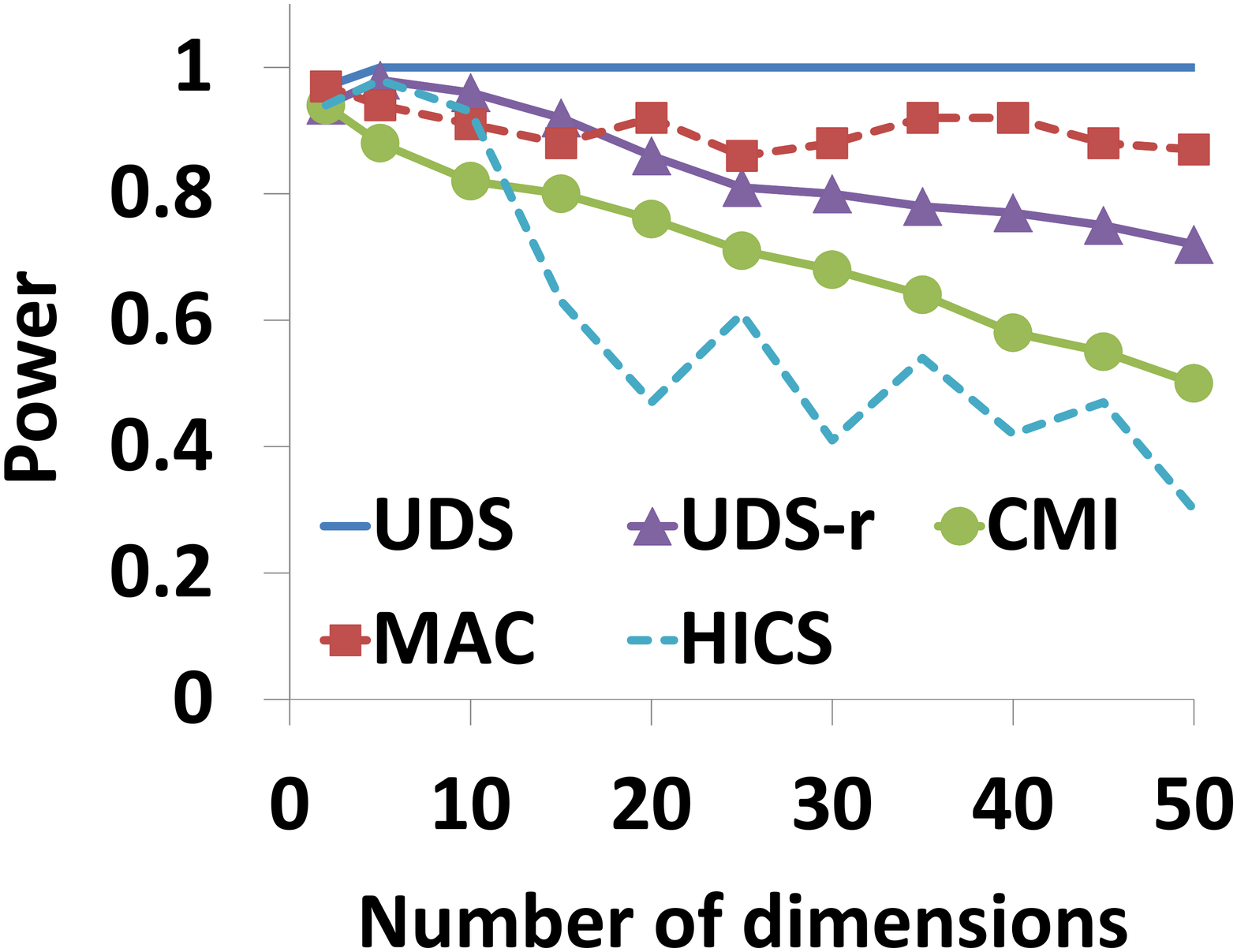}}\label{fig:power_f4}}
\caption{[Higher is better] Statistical power on synthetic data sets for the setting where $\size = 4000$ and $\dima$ is varied.} \label{fig:power}
\end{figure}

\begin{figure}[tb]
\centering
\subfigure[Power on $f_1$]
{{\includegraphics[width=0.23\textwidth]{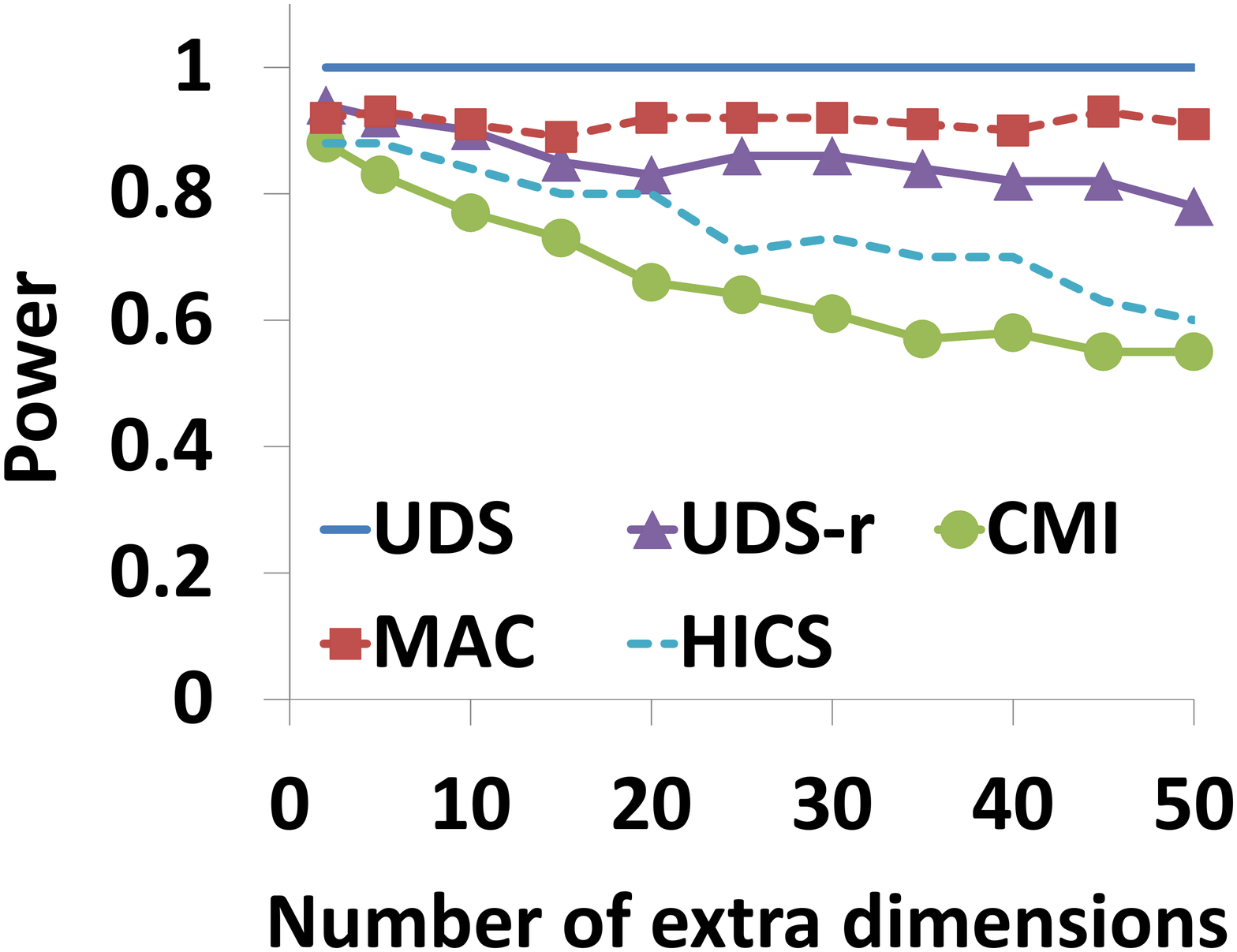}}\label{fig:power_f1_extra}}~
\subfigure[Power on $f_2$]
{{\includegraphics[width=0.23\textwidth]{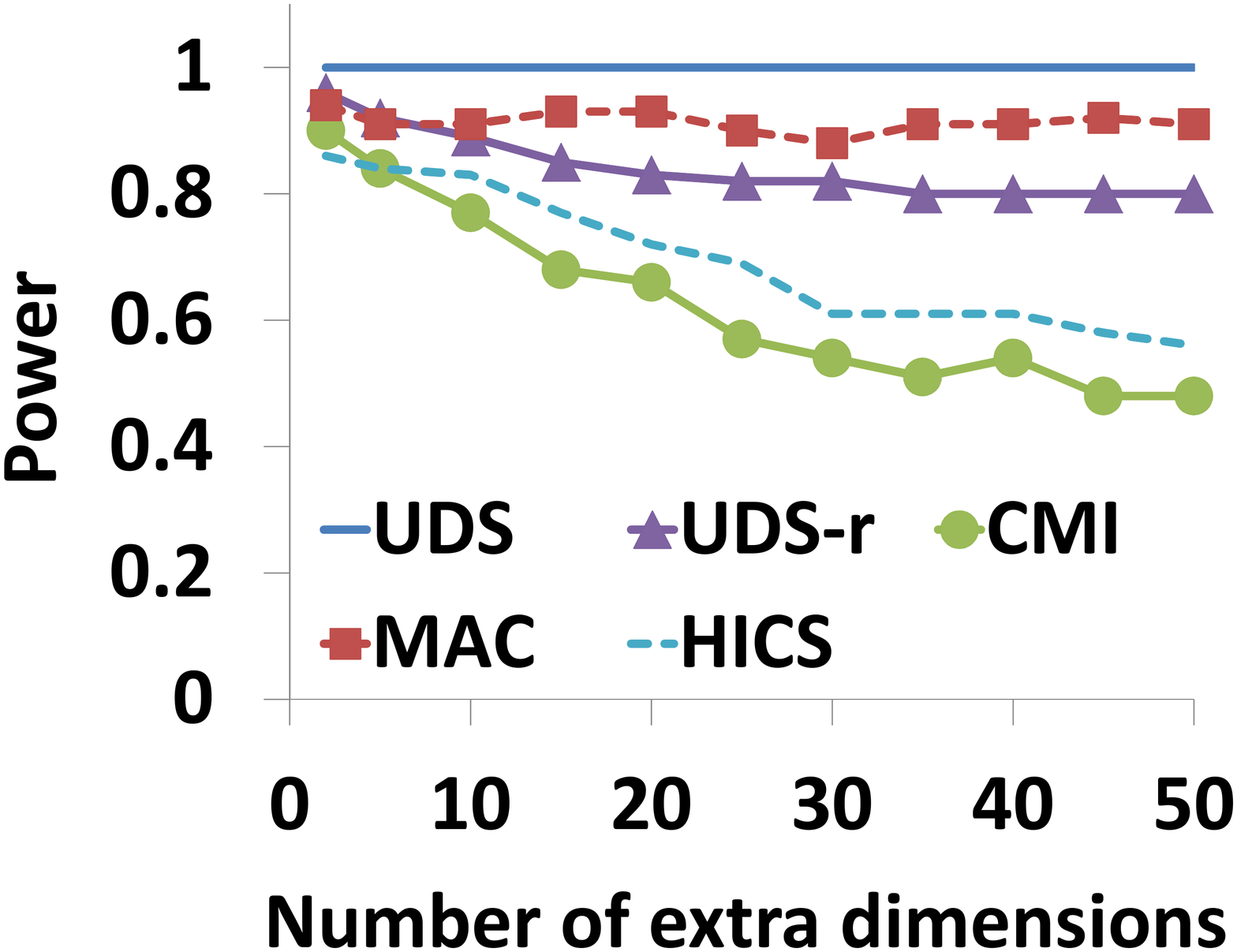}}\label{fig:power_f2_extra}}
\subfigure[Power on $f_3$]
{{\includegraphics[width=0.23\textwidth]{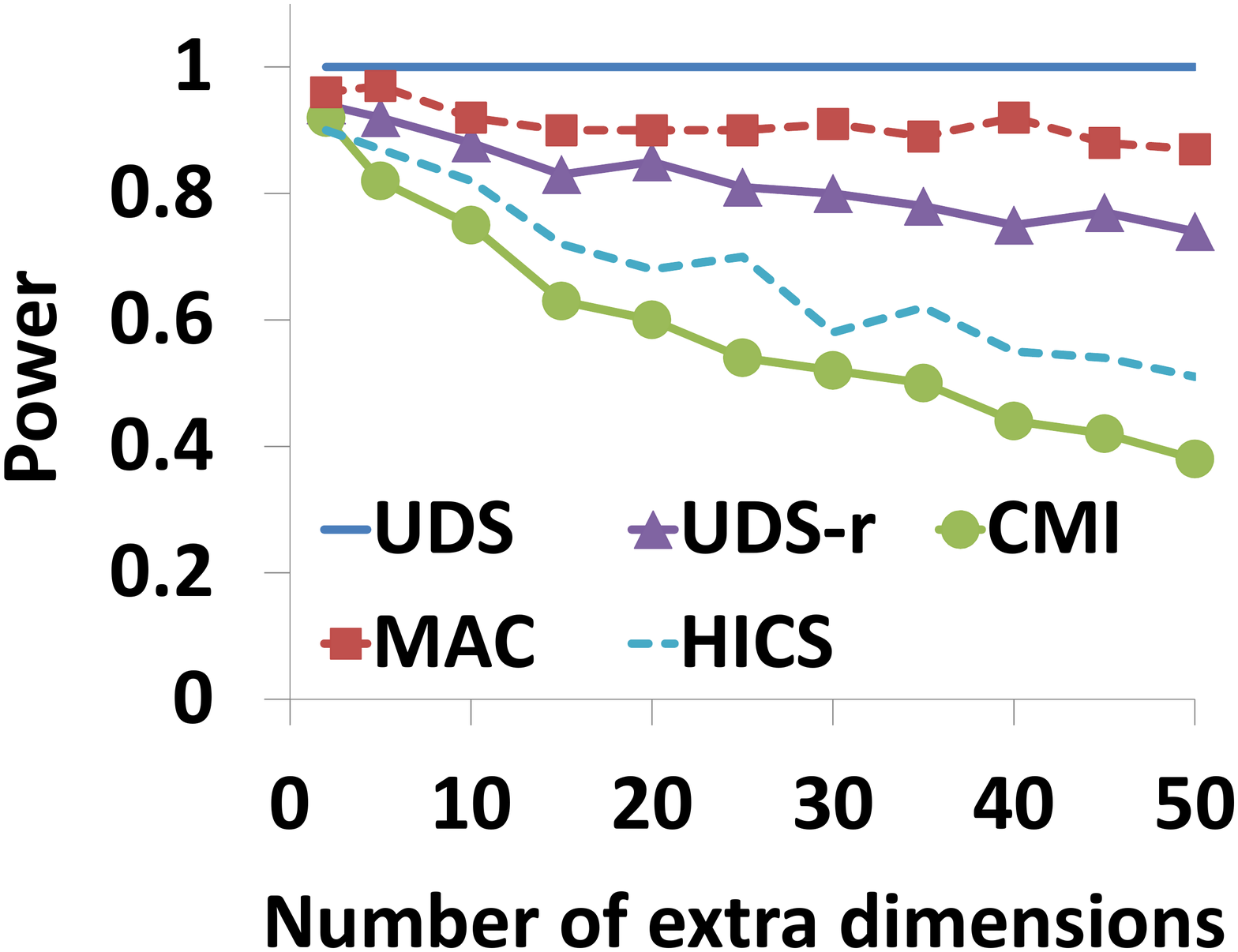}}\label{fig:power_f3_extra}}~
\subfigure[Power on $f_3$]
{{\includegraphics[width=0.23\textwidth]{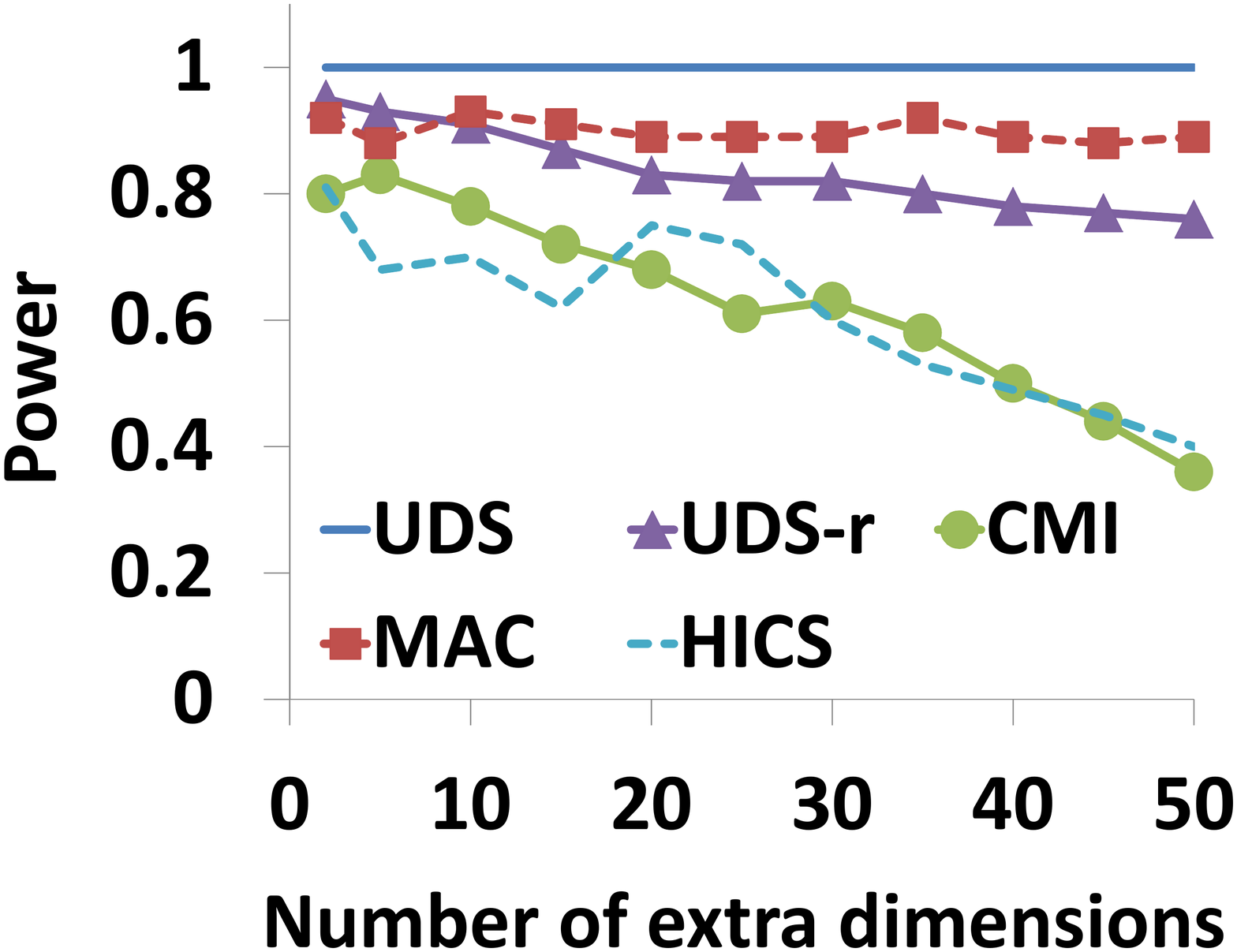}}\label{fig:power_f4_extra}}
\caption{[Higher is better] Statistical power on synthetic data sets for the setting where $\size = 4000$, $\dima = 20$, and $e$ is varied.} \label{fig:power_extra}
\end{figure}

Going over all results, we see that \uds consistently achieves the best performance in all cases. Moreover, it has from almost perfect to perfect statistical power across different values of data size $\size$, dimensionality $\dima$, and the number of extra dimensions $e$. We outperform \mac most likely because we consistently stick to cumulative entropy instead of transitioning between this entropy notion and Shannon entropy. The clear margin of \uds over \udsr shows the importance of our model selection step (cf., Section~\ref{sec:model}). \hics does not perform so well in high dimensionality perhaps due to its use of high dimensional conditional distributions.

Regarding efficiency, \uds is much faster than \mac and on par with \udsr, \cmi, and \hics. More details are in Appendix~\ref{sec:efficiencyresults}. As \uds clearly outperforms \udsr, we skip \udsr in the following.

\subsection{Quality of Subspaces -- Quantitative Results}

Here we plug all methods into beam search~\cite{keller:hics} to find correlated subspaces. As common in subspace search~\cite{cheng:enclus,nguyen:cmi,nguyen:mac}, we evaluate quality of output subspaces by each method through clustering, which tends to yield meaningful results on subspaces with high correlations~\cite{cheng:enclus,mueller:evaluating}.

For each method, we apply \dbscan~\cite{ester:dbscan} -- a well-known clustering technique -- \textit{on top} of its output subspaces. We follow~\cite{assent:dusc} to aggregate the results of all subspaces. We experiment with 6 real labeled data sets from UCI Repository, regarding their class labels as ground truth. As performance metric, we use F1 measure. The results are in Table~\ref{tab:clustering}. We see that \uds performs very well, achieving the best F1 scores on all data sets. This implies that \uds finds better correlated subspaces that help \dbscan to more accurately discover true clusters.

To further evaluate quality of subspaces found by each method, we also perform outlier detection. The results, put in Appendix~\ref{sec:outlierresults}, show that \uds also outperforms all competitors.

\begin{table}[tb]
\centering 
\begin{tabular}{lrrrr}
\toprule
Data & {\bf \uds} & {\bf \cmi} & {\bf \mac} & {\bf \hics}\\
\otoprule

Optical & \textbf{0.61} & 0.40 & 0.48 & 0.36\\

Leaves & \textbf{0.70} & 0.52 & 0.61 & 0.45\\

Letter & \textbf{0.82} & 0.64 & \textbf{0.82} & 0.49\\

PenDigits & \textbf{0.85} & 0.72 & \textbf{0.85} & 0.71\\

Robot & \textbf{0.54} & 0.33 & 0.46 & 0.21\\

Wave & \textbf{0.50} & 0.24 & 0.38 & 0.18\\

\midrule
		
Average & \textbf{0.67} & 0.48 & 0.60 & 0.40\\

\bottomrule
\end{tabular}
\caption{[Higher is better] Clustering results (F1 scores) on real-world data sets.} \label{tab:clustering}
\end{table}

\subsection{Quality of Subspaces -- Qualitative Results}

To evaluate the efficacy of \uds in exploratory analysis, we apply it on two real data sets: Communities \& Crime from demographic domain~\cite{redmond:crime} and Energy from architecture domain~\cite{wagner:building}. As these data sets are unlabeled, we cannot assess clustering quality as before. We instead perform subspace search to detect correlated subspaces, and investigate the discovered correlations. We present some interesting correlations discovered by \uds. All reported correlations are significant at $\alpha = 0.05$ following the testing procedure in~\cite{ReshefEtAl2011}.

On Communities \& Crime, \uds finds a multivariate correlation among \% of people in the community with high education, \% with low education, \% employed as worker, and \% employed as manager. This correlation is intuitively understandable. Surprisingly, it is \textit{not} detected by methods other than \uds and \mac. For exposition, we plot some of its 2-D projections in Figure~\ref{fig:demo}. For each correlation pattern in this figure, we plot the function that best fit it -- in term of $R^2$. Two out of three functions are polynomials of degree 5, implying the respective correlation patterns are non-linear.

On Energy, \uds identified a multivariate correlation between outdoor temperature, indoor CO$_2$ concentration, heating consumption, and drinking water consumption. We plot some of its 2-D projections in Figure~\ref{fig:climate}. The correlation patterns there range from linear (Figure~\ref{fig:co2_water}) to non-linear (Figures~\ref{fig:heating_water} and~\ref{fig:outdoor_water}). They are also intuitively understandable. Interestingly, no competitor including \mac can detect all of them. For instance, \mac is able to identify the first pattern only. This could be attributed to the fact that \mac computes Shannon entropy over discretized data that it obtains by optimizing w.r.t.\ cumulative entropy.

\begin{figure*}[tb]
\centering
\subfigure[\% high education vs.\ \% employed as worker]
{{\includegraphics[width=0.3\textwidth]{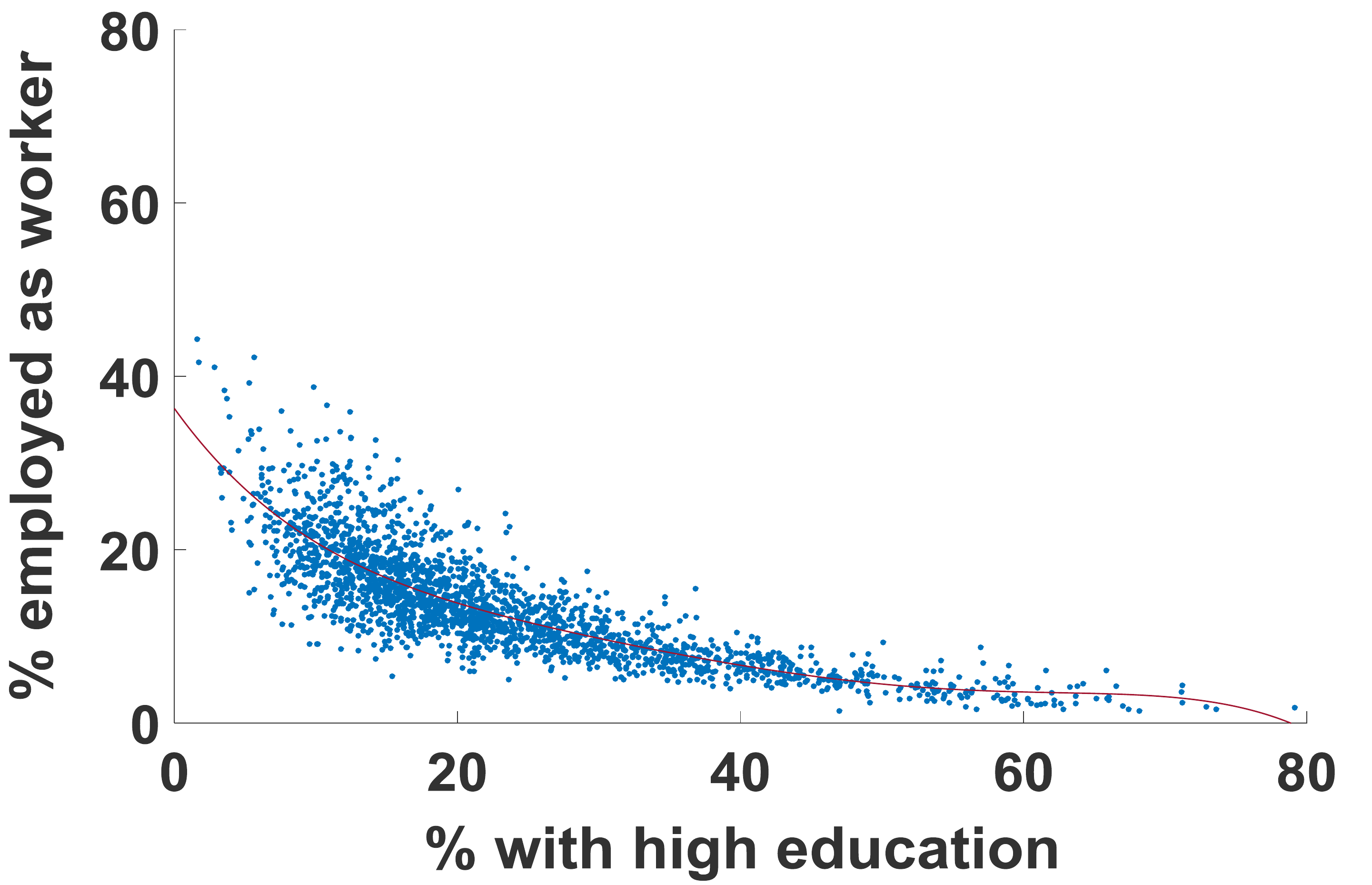}}\label{fig:bach_worker}}~
\subfigure[\% high education vs.\ \% employed as manager]
{{\includegraphics[width=0.3\textwidth]{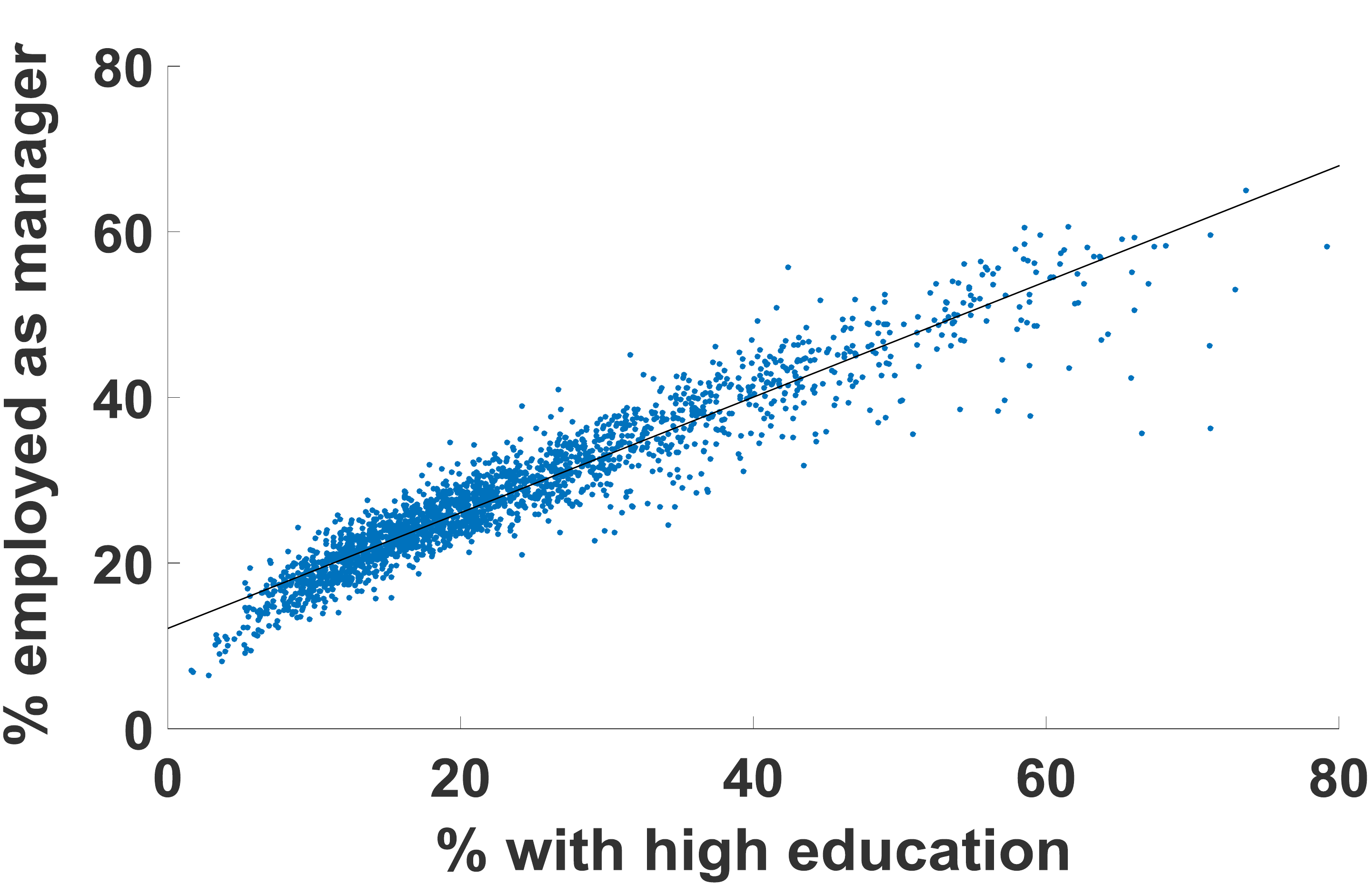}}\label{fig:bach_mng}}~
\subfigure[\% low education vs.\ \% employed as manager]
{{\includegraphics[width=0.3\textwidth]{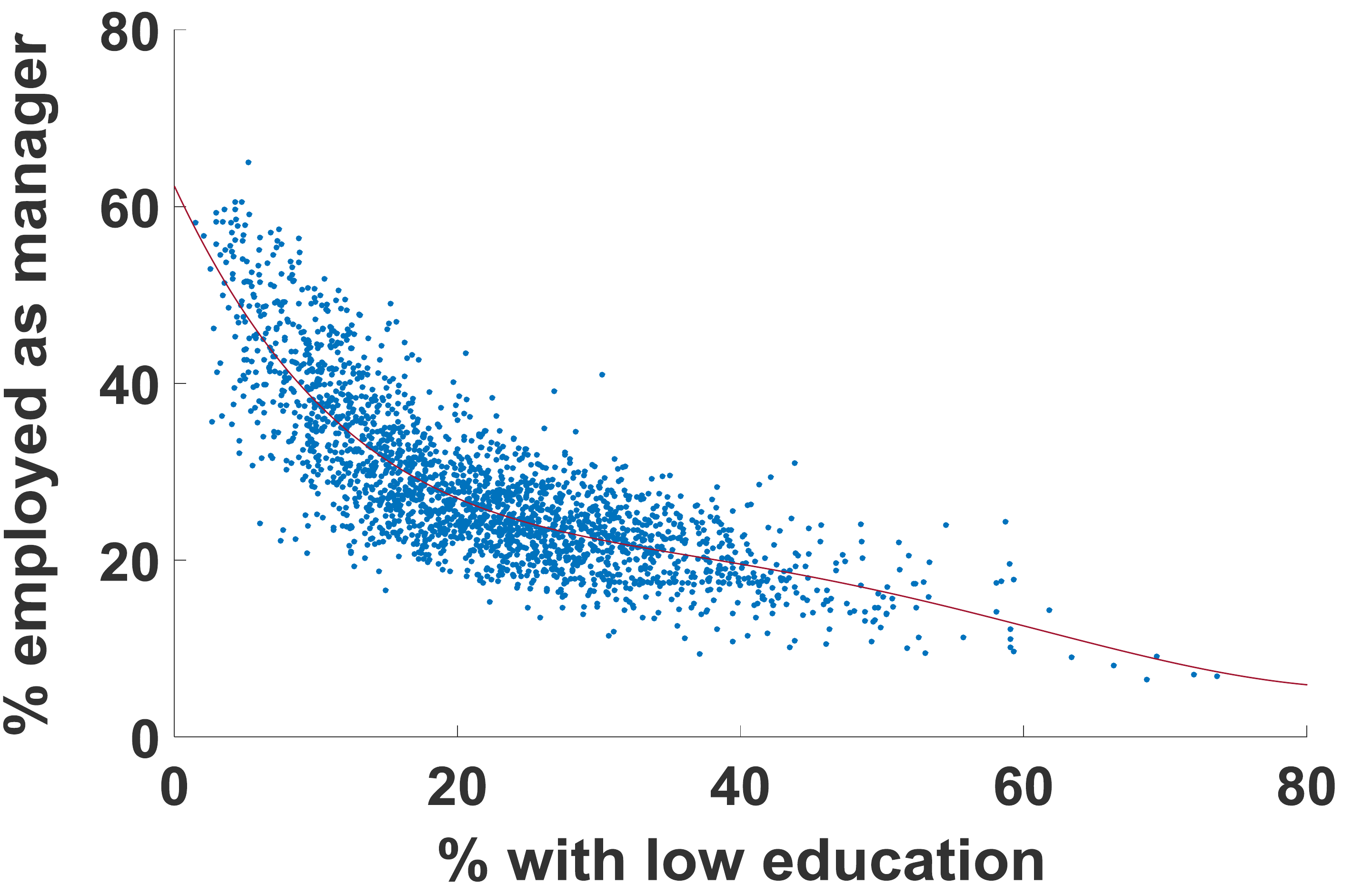}}\label{fig:low_mng}}
\caption{Examples of correlation found by \uds on demographic data. Of the competitors, only \mac can discover these correlation patterns.} \label{fig:demo}
\end{figure*}

\begin{figure*}[tb]
\centering
\subfigure[water consumption vs.\ CO$_2$ concentration]
{{\includegraphics[width=0.3\textwidth]{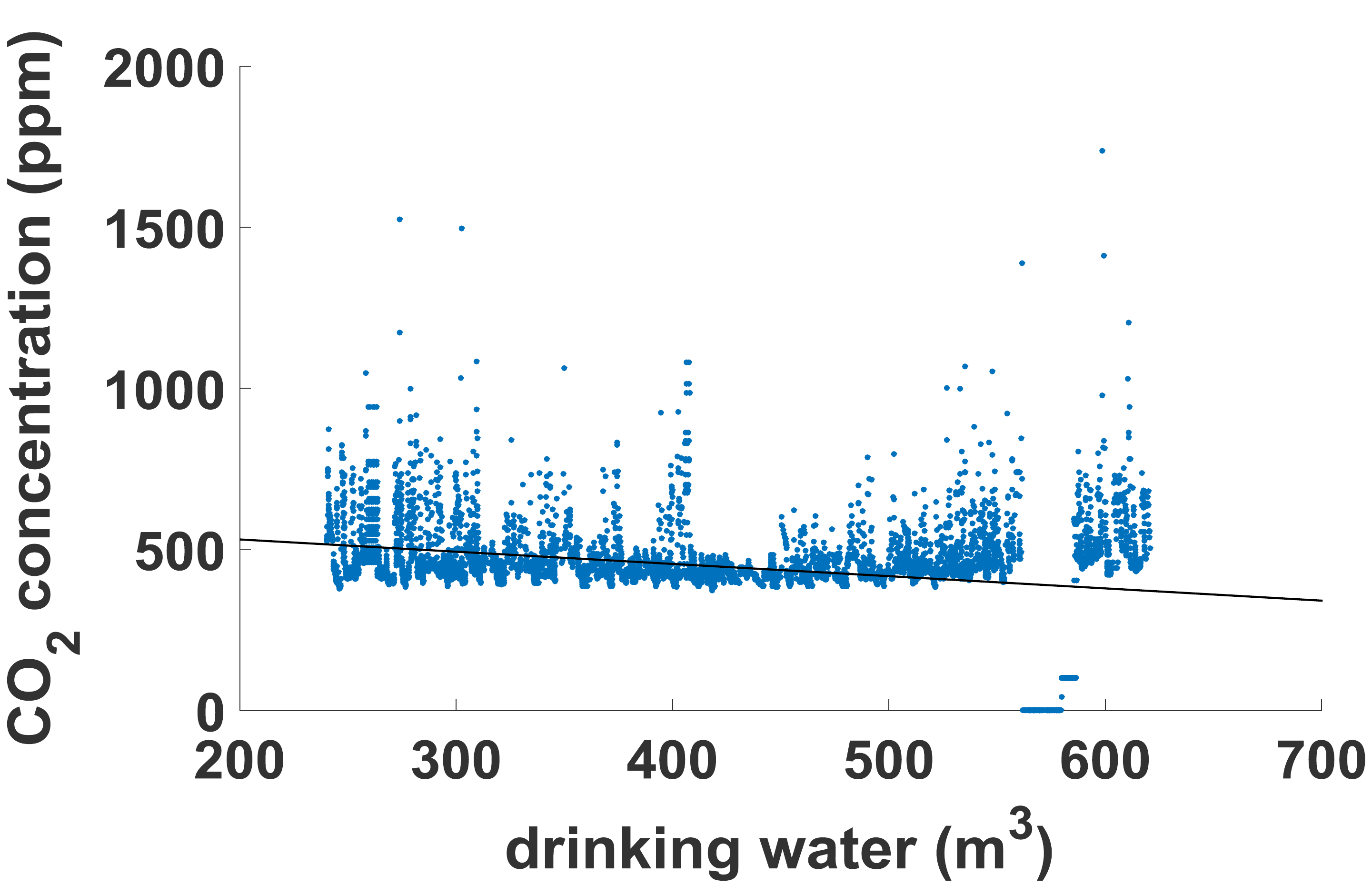}}\label{fig:co2_water}}~
\subfigure[heating consumption vs.\ water consumption]
{{\includegraphics[width=0.3\textwidth]{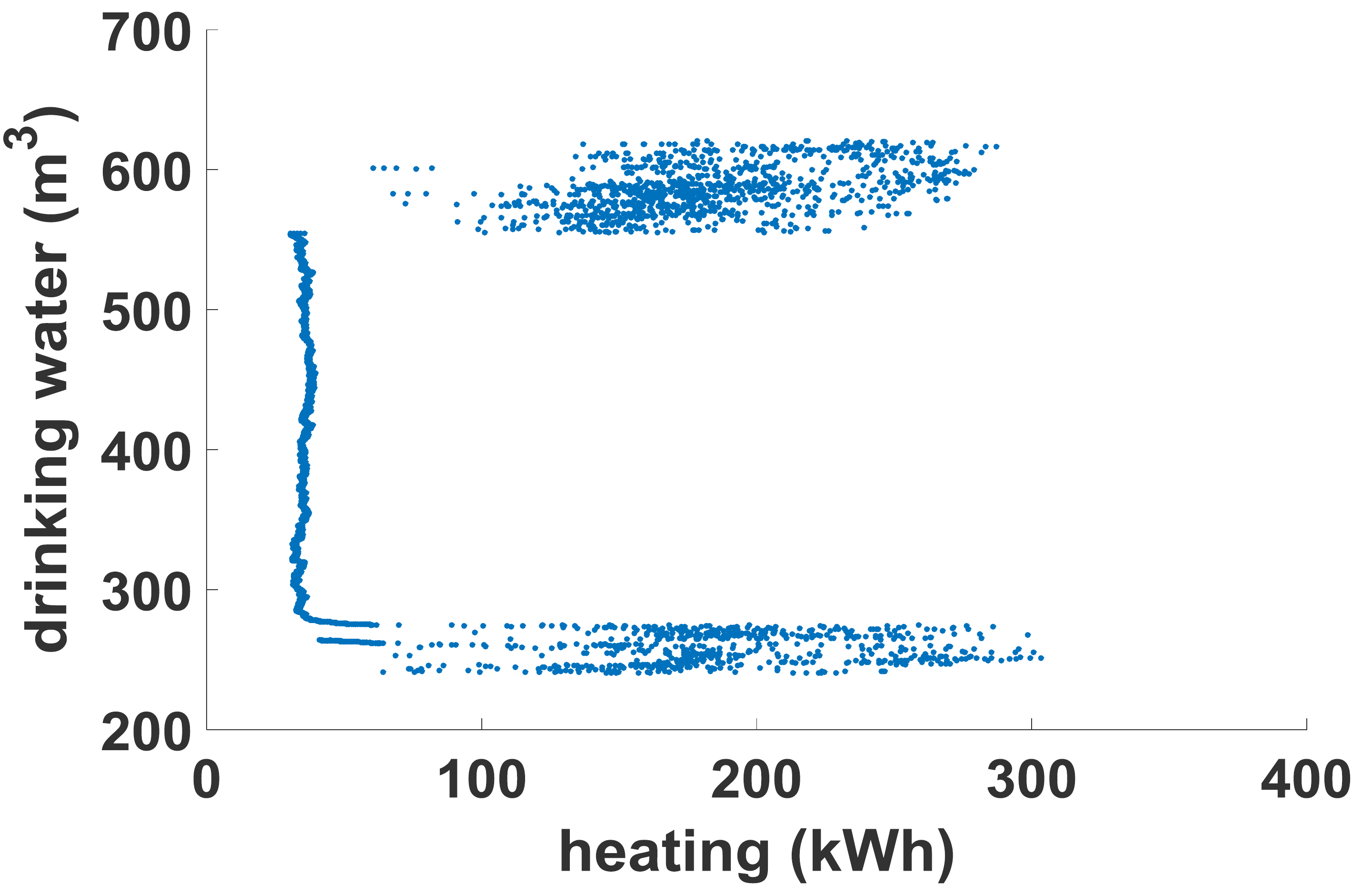}}\label{fig:heating_water}}~
\subfigure[water consumption vs.\ outdoor temperature]
{{\includegraphics[width=0.3\textwidth]{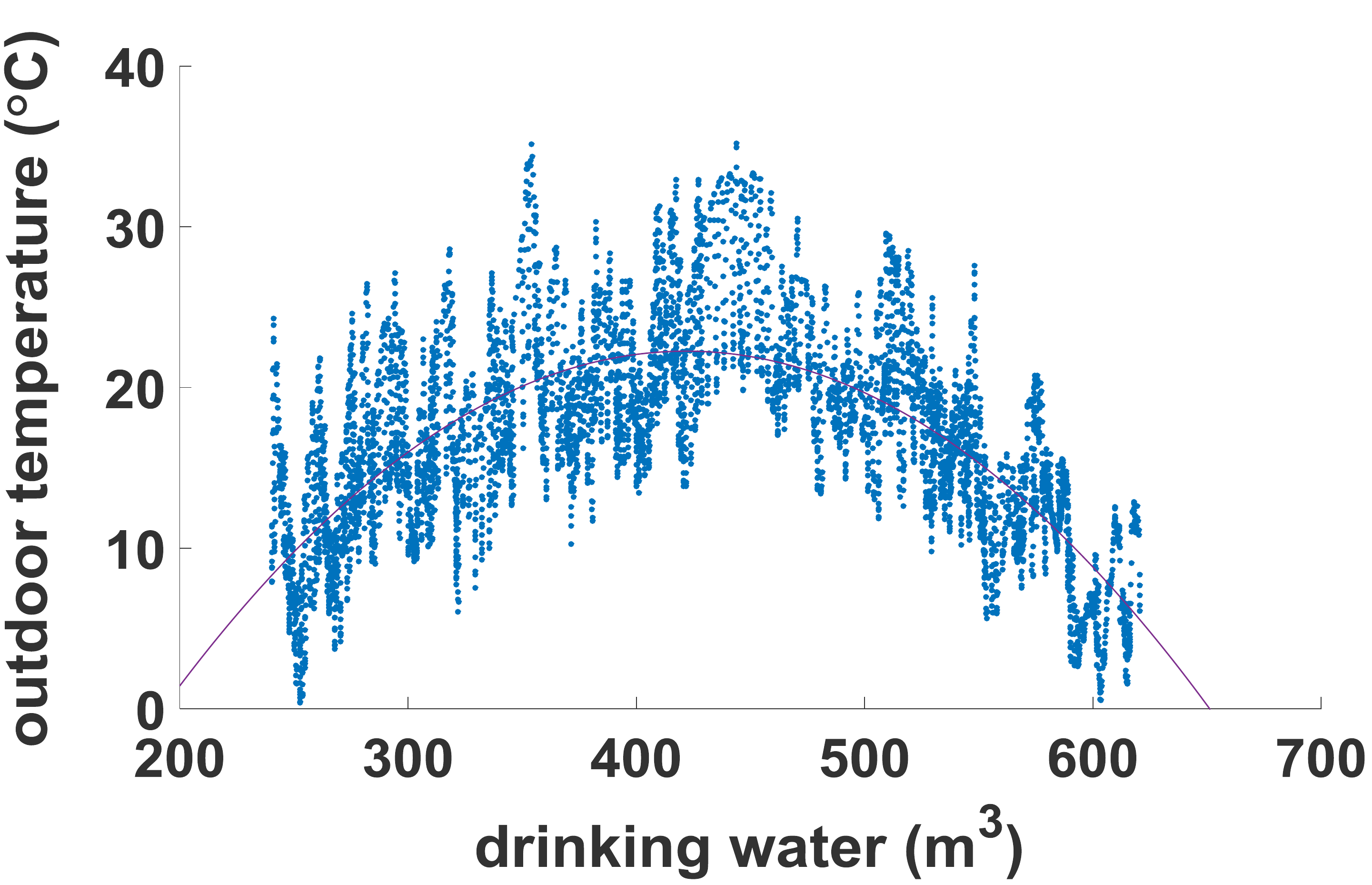}}\label{fig:outdoor_water}}
\caption{Examples of correlation found by \uds on climate data. Of the competitors, only \mac can discover the first correlation pattern.} \label{fig:climate}
\end{figure*}

%% file: 08con.tex
\section{Conclusion} \label{sec:con}

In this paper, we studied the problem of universally, non-parametrically, and efficiently assessing subspace correlations in multivariate data. By universal, we mean that 1) we are able to capture correlation in subspaces of any dimensionality and 2) we allow comparison of correlation scores across different subspaces -- regardless how many dimensions they have and what specific statistical properties their dimensions possess. To address all issues, we proposed \uds. In short, we defined \uds based on cumulative entropy. We fulfilled universality by introducing a principled normalization scheme to bring \uds scores across different subspaces to the same domain. We presented a non-parametric and efficient method to compute \uds on empirical data. Extensive experiments showed that \uds outperformed state of the art in both statistical power and subspace search.

%% file: appendix.tex
\section{Proofs} \label{sec:proofs}

\begin{proof}[Lemma~\ref{lem:cmibias}]
By definition, we have
\begin{align*}
& \textstyle\cmi(X_{1, \ldots, d}) \\
& = \max\limits_{\sigma \in \mathcal{F}_d} \sum\limits_{i=2}^d h(X_{\sigma(i)}) - h(X_{\sigma(i)} \mid X_{\sigma(1), \ldots, \sigma(i-1)}).
\end{align*}
For each $\sigma \in \mathcal{F}_d$, 
\begin{align*}
& \sum\limits_{i=2}^d h(X_{\sigma(i)}) - h(X_{\sigma(i)} \mid X_{\sigma(1), \ldots, \sigma(i-1)}) \\
& \leq \sum\limits_{i=2}^d h(X_{\sigma(i)}) - h(X_{\sigma(i)} \mid X_{\sigma(1), \ldots, \sigma(i-1)})\\
& \qquad \qquad + h(X_{d+1}) - h(X_{d+1} \mid X_{\sigma(1), \ldots, \sigma(i)}).
\end{align*}
This holds because from Theorem~\ref{theo:cenonneg},
$$h(X_{d+1}) \geq h(X_{d+1} \mid X_{\sigma(1), \ldots, \sigma(i)}).$$
Hence, for each $\sigma \in \mathcal{F}_d$,
\begin{align*}
& \sum\limits_{i=2}^d h(X_{\sigma(i)}) - h(X_{\sigma(i)} \mid X_{\sigma(1), \ldots, \sigma(i-1)}) \\
& \leq \max\limits_{\sigma \in \mathcal{F}_{d+1}} \sum\limits_{i=2}^{d+1} h(X_{\sigma(i)}) - h(X_{\sigma(i)} \mid X_{\sigma(1), \ldots, \sigma(i-1)}).
\end{align*}
In other words, $\cmi(X_{1, \ldots, d}) \leq \cmi(X_{1, \ldots, d + 1})$.
\end{proof}

\begin{proof}[Lemma~\ref{lem:bound1}]
From Theorem~\ref{theo:cenonneg}, it holds that
\begin{align*}
& \sum\limits_{i=2}^d h(X_{\sigma(i)}) - h(X_{\sigma(i)} \mid X_{\sigma(1), \ldots, \sigma(i-1)}) \\
& \leq \sum\limits_{i=2}^d h(X_{\sigma(i)}) \\
& \leq \max\limits_{\sigma \in \mathcal{F}_d} \sum\limits_{i=2}^d h(X_{\sigma(i)}).
\end{align*}
Hence, we arrive at: $\cmi(X_{1, \ldots, d}) \leq \max\limits_{\sigma \in \mathcal{F}_d} \sum\limits_{i=2}^d h(X_{\sigma(i)})$.
\end{proof}

\begin{proof}[Lemma~\ref{lem:bound}]
The result follows from the proof of Lemma~\ref{lem:cmibias}. Based on, Theorem~\ref{theo:cenonneg}, the equality holds iff for each $i \in [2, d]$, $X_{\sigma(i)}$ is a function of $X_{\sigma(1), \ldots, \sigma(i-1)}$. This means that
\begin{itemize}
\item $X_{\sigma(2)}$ is a function of $X_{\sigma(1)}$.
\item $X_{\sigma(3)}$ is a function of $X_{\sigma(1)}$ and $X_{\sigma(2)}$.
\item \ldots
\item $X_{\sigma(d)}$ is a function of $X_{\sigma(1), \ldots, \sigma(d-1)}$.
\end{itemize}
This is equivalent to that $X_{\sigma(2), \ldots, \sigma(d)}$ are functions of $X_{\sigma(1)}$.
\end{proof}

\begin{proof}[Lemma~\ref{lem:udsbound}]
That $0 \leq \uds(X_{1, \ldots, d}) \leq 1$ follows from Theorem~\ref{theo:cenonneg} and Lemma~\ref{lem:bound}.

$\uds(X_{1, \ldots, d}) = 0$ iff for each $\sigma \in \mathcal{F}_d$,
$$\sum\limits_{i=2}^d h(X_{\sigma(i)}) - h(X_{\sigma(i)} \mid X_{\sigma(1), \ldots, \sigma(i-1)}) = 0.$$
This means that for each $i \in [2, d]$, $X_{\sigma(i)}$ is independent of $X_{\sigma(1), \ldots, \sigma(i-1)}$, which is equivalent to $X_{\sigma(1), \ldots, \sigma(d)}$ are independent.

$\uds(X_{1, \ldots, d}) = 0$ iff there exists $\sigma \in \mathcal{F}_d$ such that
$$\sum\limits_{i=2}^d h(X_{\sigma(i)}) - h(X_{\sigma(i)} \mid X_{\sigma(1), \ldots, \sigma(i-1)}) = \sum\limits_{i=2}^d h(X_{\sigma(i)}).$$
Following Lemma~\ref{lem:bound}, this means that $X_{\sigma(2), \ldots, \sigma(d)}$ are functions of $X_{\sigma(1)}$.
\end{proof}

\begin{lemma} \label{lem:pracudsbound}
It holds that:\\
$\bullet$ $0 \leq \uds_{\pr}(X_{1, \ldots, d}) \leq 1$.\\
$\bullet$ $\uds_{\pr}(X_{1, \ldots, d}) = 0$ iff $X_{1, \ldots, d}$ are independent.\\
$\bullet$ $\uds_{\pr}(X_{1, \ldots, d}) = 1$ iff $X_{\sigma(2), \ldots, \sigma(d)}$ are functions of $X_{\sigma(1)}$.
\end{lemma}

\begin{proof}
That $0 \leq \uds_{\pr}(X_{1, \ldots, d}) \leq 1$ follows from Theorem~\ref{theo:cenonneg} and Lemma~\ref{lem:bound}.

$\uds_{\pr}(X_{1, \ldots, d}) = 0$ iff
$$\sum\limits_{i=2}^d h(X_{\sigma(i)}) - h(X_{\sigma(i)} \mid X_{\sigma(1), \ldots, \sigma(i-1)}) = 0.$$
This means that for each $i \in [2, d]$, $X_{\sigma(i)}$ is independent of $X_{\sigma(1), \ldots, \sigma(i-1)}$, which is equivalent to $X_{\sigma(1), \ldots, \sigma(d)}$ are independent.

$\uds_{\pr}(X_{1, \ldots, d}) = 1$ iff
$$\sum\limits_{i=2}^d h(X_{\sigma(i)}) - h(X_{\sigma(i)} \mid X_{\sigma(1), \ldots, \sigma(i-1)}) = \sum\limits_{i=2}^d h(X_{\sigma(i)}).$$
Following Lemma~\ref{lem:bound}, this means that $X_{\sigma(2), \ldots, \sigma(d)}$ are functions of $X_{\sigma(1)}$.
\end{proof}

\section{Detailed Complexity Analysis} \label{sec:complex}

We pre-sort the values of $X_1, \ldots, X_\dima$. For each dimension $X$, we maintain a \textit{rank index} structure $\mathbf{RI}_X$ where $\mathbf{RI}_X[i]$ is the row index of the $\mathit{i^{th}}$ smallest value of $X$.

When computing $h(X' \mid I, X)$, we pre-partition $X$ in to bins $\{a_1, \ldots, a_\beta\}$. Recall that $C_1, \ldots, C_k$ are the non-empty hypercubes of $I$. To efficiently compute $\pref[j][i]$ where $1 \leq j \leq i \leq \beta$, we have to make sure that values of $X'$ within each combination $(C_y, a_z)$ for $y \in [1, k]$ and $z \in [1, \beta]$ are already sorted. We achieve this by using $\mathbf{RI}_{X'}$. In particular, we loop through each row index, starting from~0. For each index $i$, we obtain the real row index which is $\mathbf{RI}_{X'}[i]$. Hence, we obtain the respective row. For this row, we retrieve the respective combination $(C_y, a_z)$ that it has. Then, we add $X'$ value of the row to the list of $X'$ values of $(C_y, a_z)$. In this way, we have values of $X'$ sorted in \textit{all} combinations $(C_y, a_z)$.

Now, for each $j \in [1, \beta]$, we run $i$ from $j$ to $\beta$. If $i = j$, $f[j][i]$ is in fact $h(X' \mid I)$ computed using the data points in bin $a_i$. If $i > j$, we compute $f[j][i]$ by merging lists of sorted values $X'$ of $(C_1, a_i), \ldots, (C_k, a_i)$ with the respective list of $(C_1, \bigcup_{z=j}^{i-1} a_z), \ldots, (C_k, \bigcup_{z=j}^{i-1} a_z)$. That is, we merge the list of $(C_y, a_i)$ with that of $(C_y, \bigcup_{z=j}^{i-1} a_z)$. The merge is done efficiently using merge procedure of the well-known merge sort.

Therefore, computing $\pref[j][i]$ where $1 \leq j \leq i \leq \beta$ in total takes $O(\size \log \size + \size \beta^2)$.

\section{Additional Results on Statistical Power} \label{sec:extraresults}

The results for case 4), where data sets with and without correlation have arbitrary dimensionality, are in Figure~\ref{fig:power_mixed}. Here, $\size = 4000$ and the dimensionality of each data set created is randomly drawn from $[2, 50]$. We see that \uds consistently achieves the best statistical power across $f_1$, $f_2$, $f_3$, and $f_4$. This means that it is able to universally quantify correlations of subspaces with different dimensionality.

\begin{figure}[tb]
\centering
\includegraphics[width=0.23\textwidth]{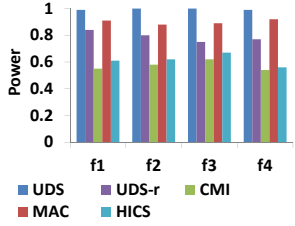}
\caption{[Higher is better] Statistical power on synthetic data sets for the setting where $\size = 4000$ and the dimensionality of each data set created is randomly drawn from $[2, 50]$.} \label{fig:power_mixed} 
\end{figure}

\section{Sensitivity to $\beta$} \label{sec:sensitivity}

To assess sensitivity of \uds to $\beta$, we use the setting of case 1) where data sets with as well as without correlation have the same dimensionality $\dima$. In particular, we generate data sets with $\size = 4000$ and $\dima = 20$. We vary $\beta$ from 5 to 40 with step size being 5. The results are in Figure~\ref{fig:beta}. We see that \uds is very stable for $\beta > 10$. This implies that it allows easy parameterization.

\begin{figure}[tb]
\centering
\subfigure[Power on $f_1$]
{{\includegraphics[width=0.23\textwidth]{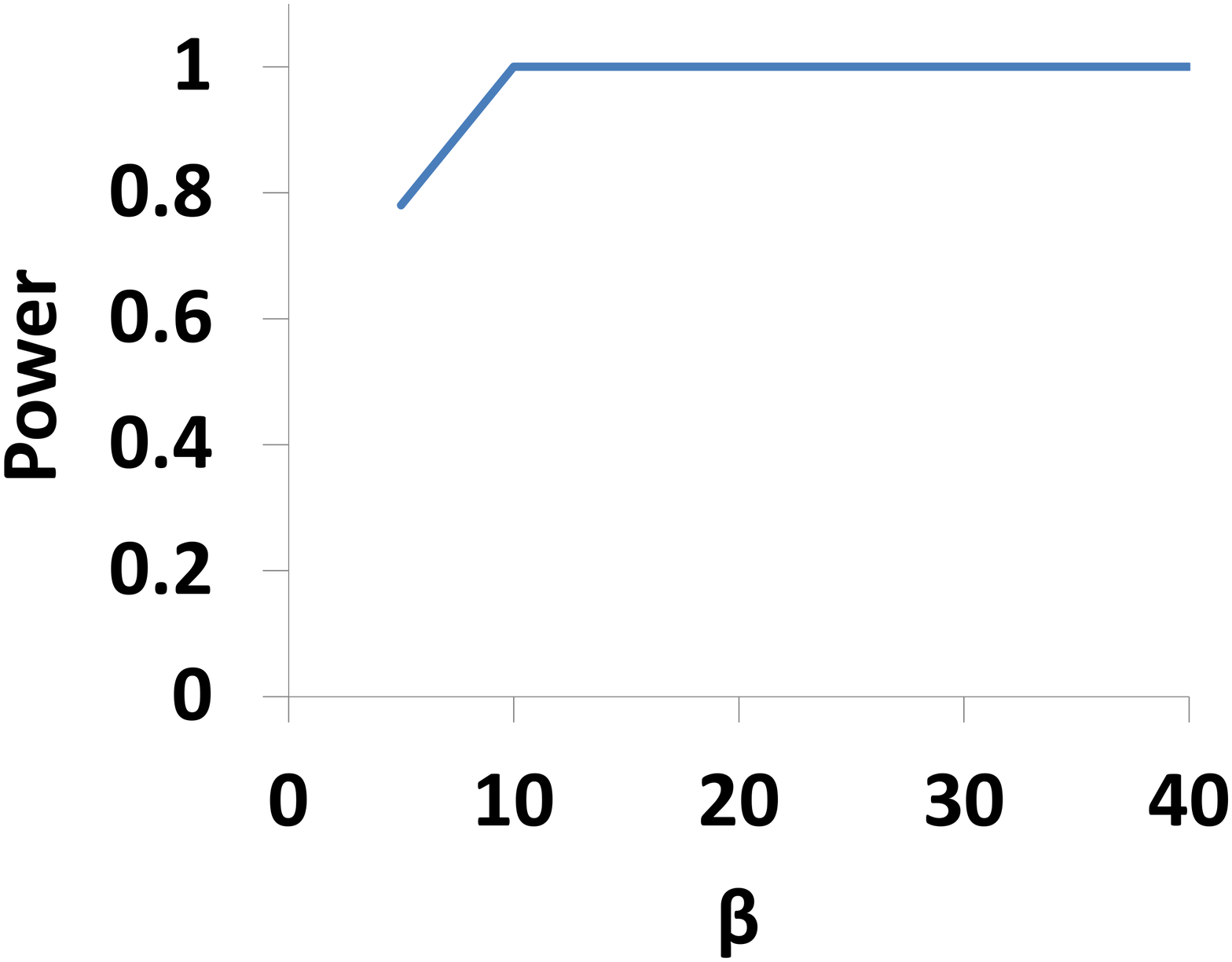}}\label{fig:f1_beta}}~
\subfigure[Power on $f_2$]
{{\includegraphics[width=0.23\textwidth]{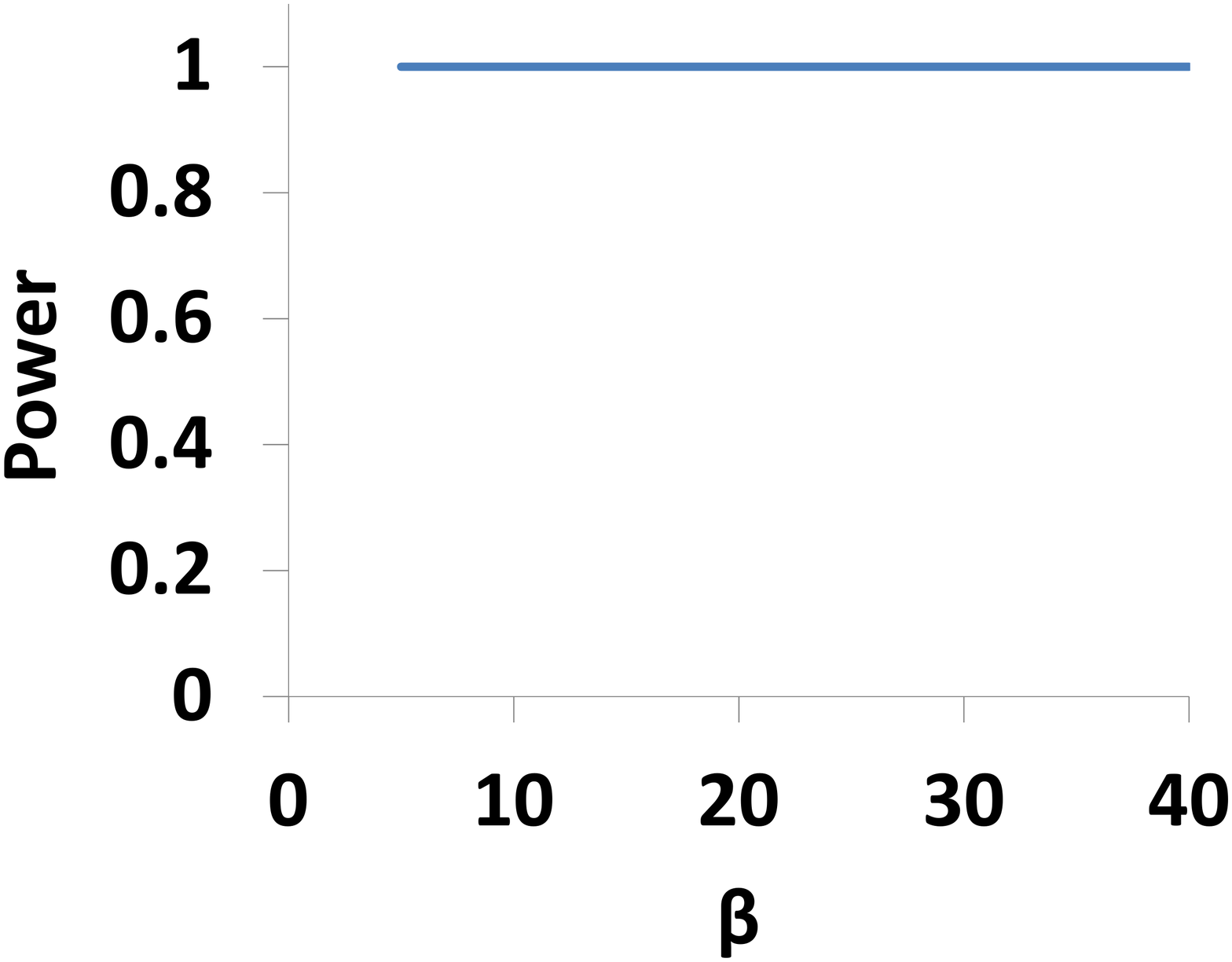}}\label{fig:f2_beta}}
\subfigure[Power on $f_3$]
{{\includegraphics[width=0.23\textwidth]{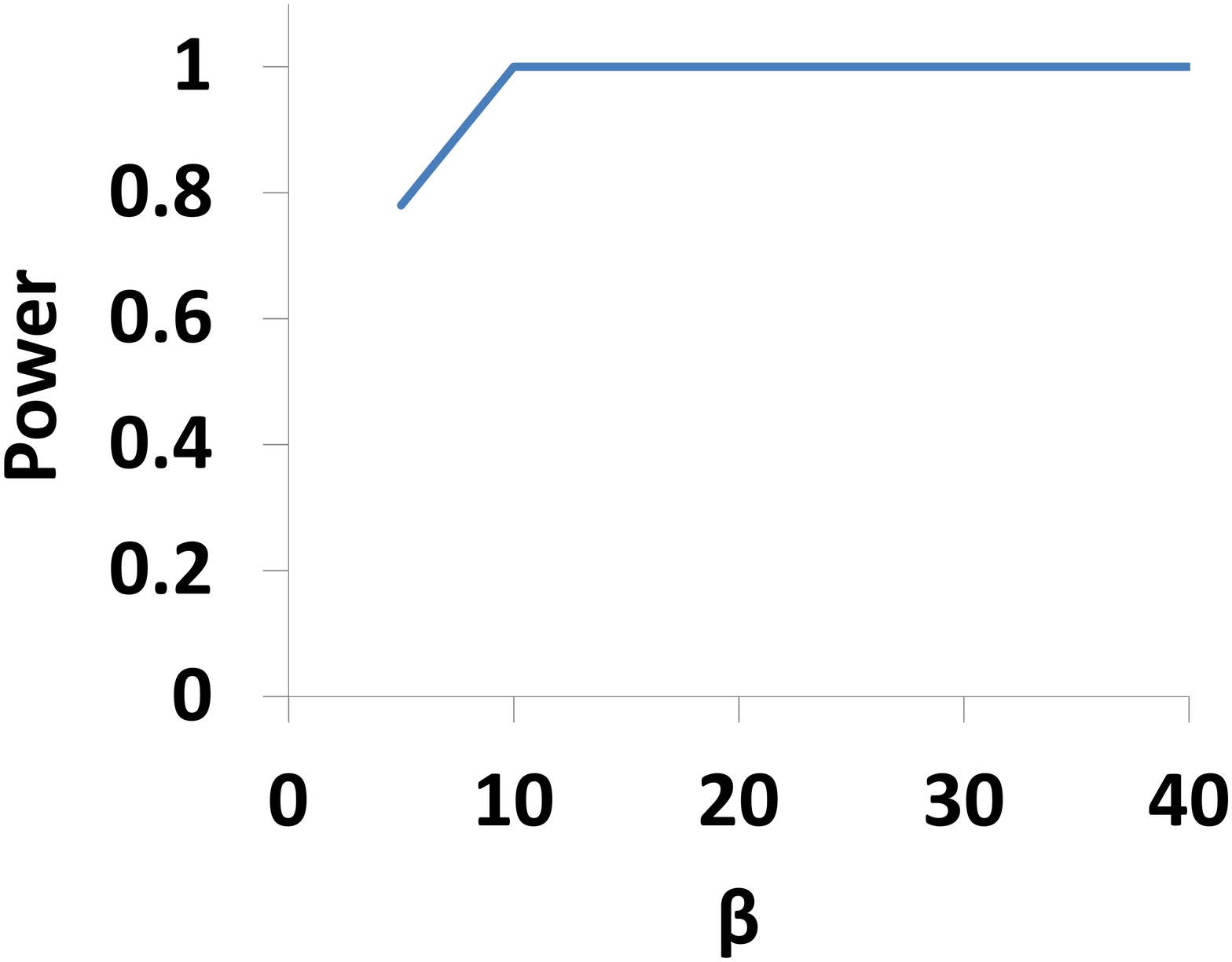}}\label{fig:f3_beta}}~
\subfigure[Power on $f_4$]
{{\includegraphics[width=0.23\textwidth]{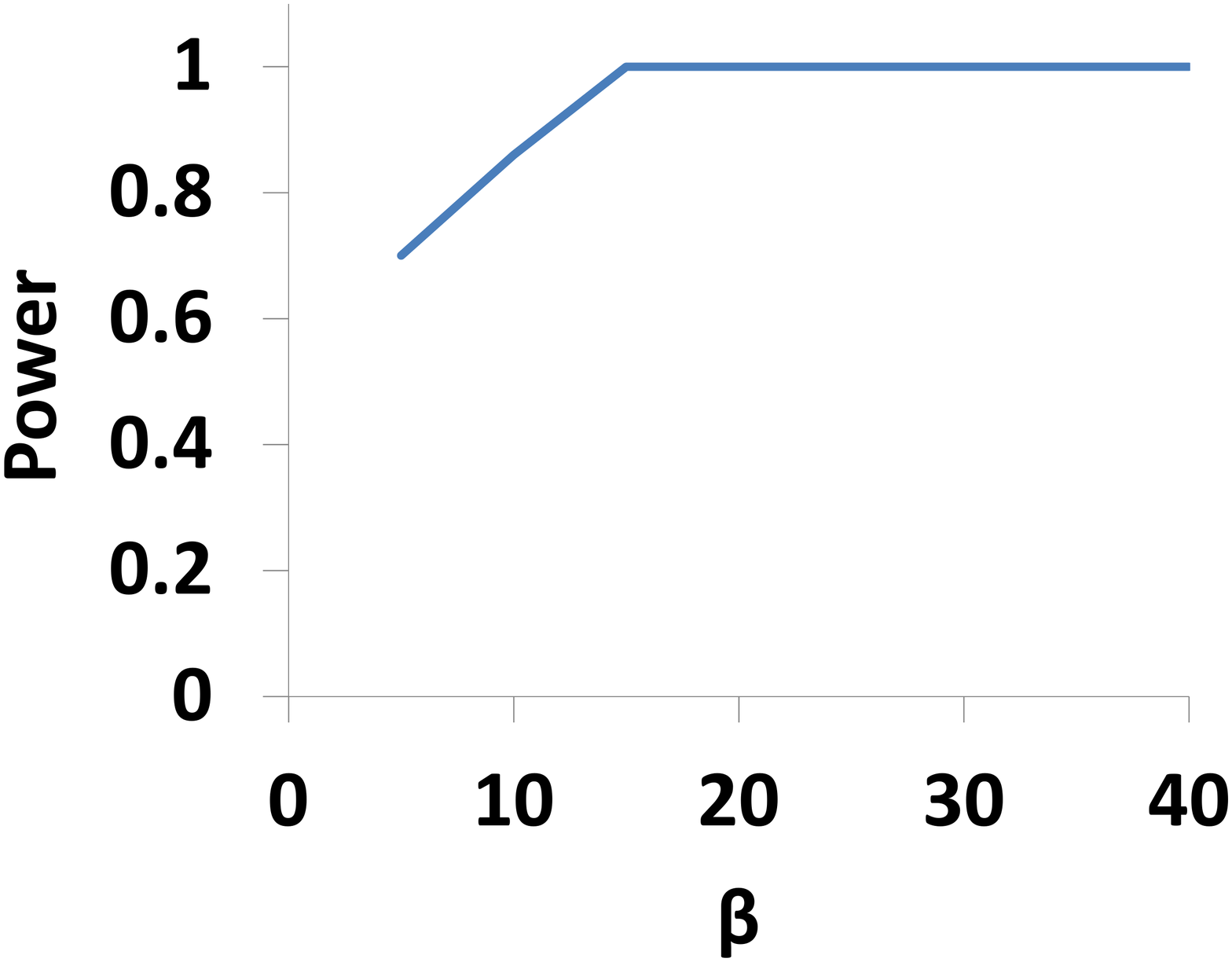}}\label{fig:f4_beta}}
\caption{[Higher is better] Sensitivity to $\beta$: Statistical power on synthetic data sets for the setting where $\size = 4000$ and $\dima = 20$.} \label{fig:beta}
\end{figure}

\section{Efficiency Results} \label{sec:efficiencyresults}

For efficiency to data size $\size$, we generate data sets with dimensionality $\dima = 20$ and $\size$ varied. For efficiency to dimensionality $\dima$, we generate data sets with size $\size = 4000$ and $\dima$ varied. The results are in Figure~\ref{fig:time}. We see that \uds scales much better than \mac and on par with other methods. \mac is inefficient because it has to compute all pairwise correlations of a subspace before outputting its final score.

\begin{figure}[tb]
\centering
\subfigure[Runtime vs.\ $\size$]
{{\includegraphics[width=0.23\textwidth]{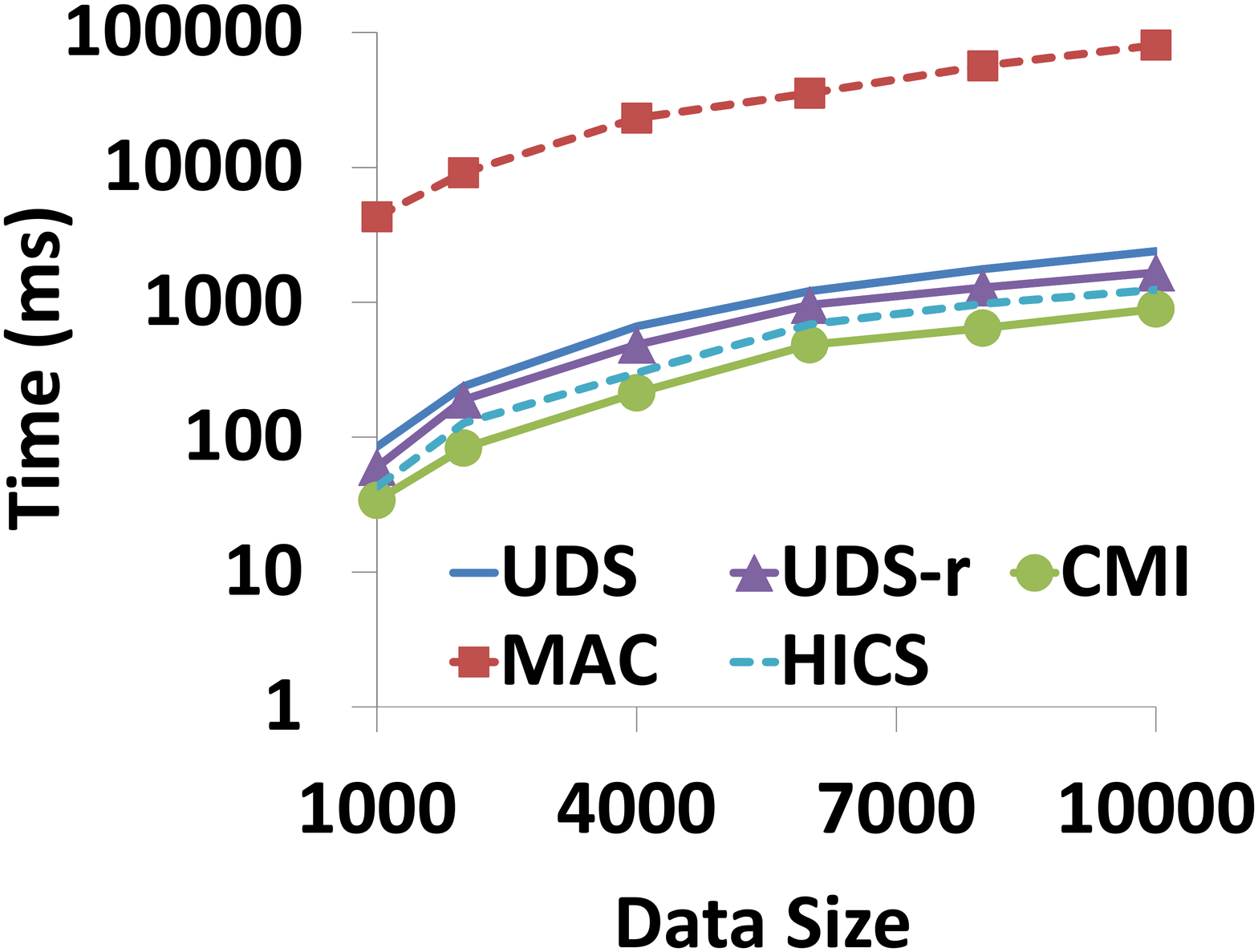}}\label{fig:time_vs_size}}~
\subfigure[Runtime vs.\ $\dima$]
{{\includegraphics[width=0.23\textwidth]{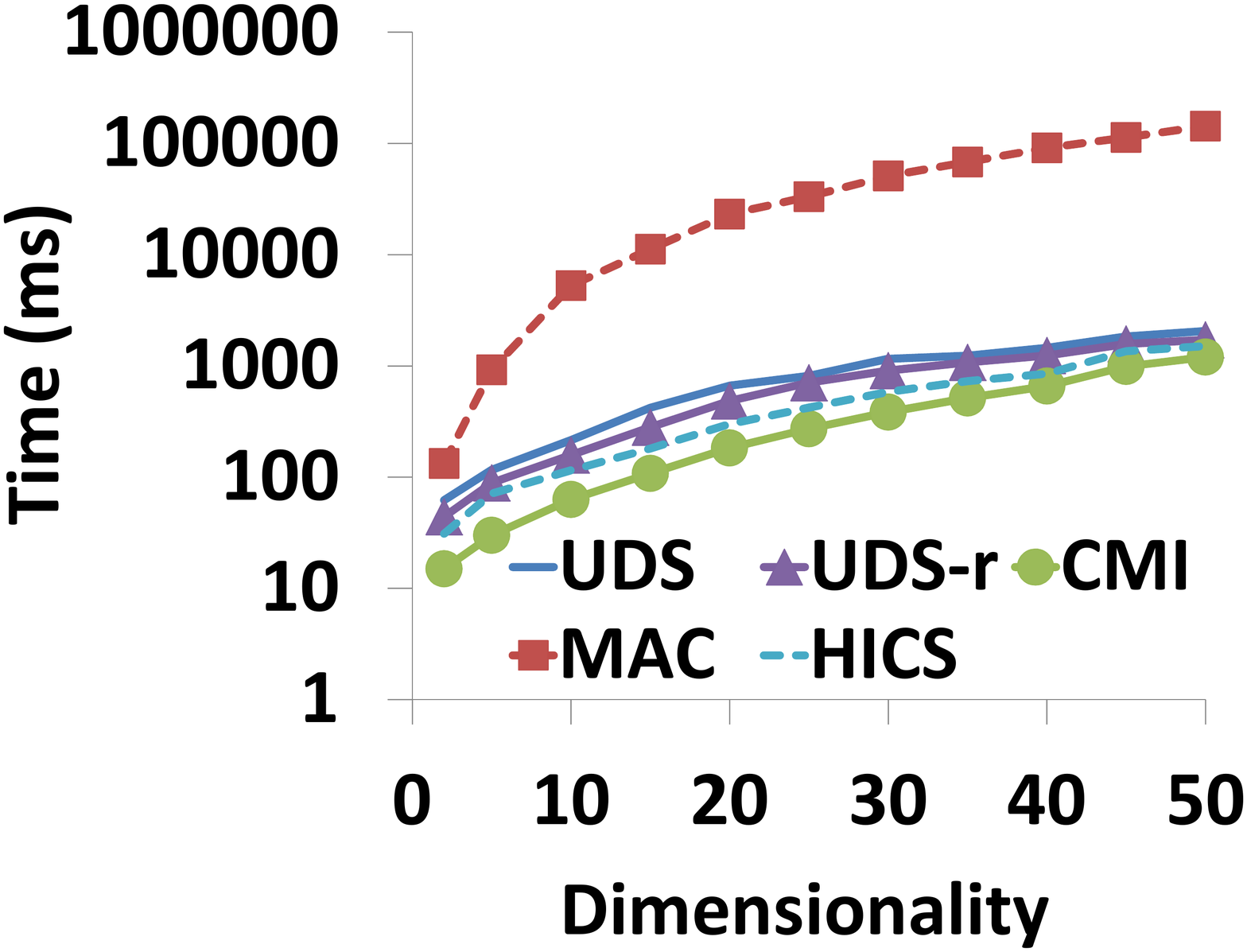}}\label{fig:time_vs_dim}}
\caption{[Lower is better] Scalability to data size $\size$ and dimensionality $\dima$. The default setting is $\size = 4000$ and $\dima = 20$. The runtime axis is in \textit{log scale}.} \label{fig:time}
\end{figure}

\section{Results on Outlier Detection} \label{sec:outlierresults}

Besides clustering, we also use outlier detection to test quality of correlated subspaces found by each method. Again, we plug all methods into beam search to find correlated subspaces. As common in subspace search~\cite{keller:hics}, for each method we apply \lof, a well-known technique for outlier detection, on top of its output subspaces. Outlier detection also tends to yield meaningful results on subspaces with high correlations~\cite{keller:hics,nguyen:cmi}.

To show that \uds can work with various data sets, we pick another 6 real labeled data sets -- also from UCI Repository -- for testing purposes. For each of these data sets, we follow standard procedure in the literature and create outliers by randomly taking 10\% of the smallest class. As performance metric, we use AUC (Area under ROC Curve). The results are in Table~\ref{tab:outlier}. We see that \uds consistently achieves the best AUC scores on all data sets. This implies that it finds better correlated subspaces that help \lof to more accurately identify true outliers.

\begin{table}[tb]
\centering 
\begin{tabular}{lrrrr}
\toprule
Data & {\bf \uds} & {\bf \cmi} & {\bf \mac} & {\bf \hics}\\
\otoprule

Ann-Thyroid & \textbf{0.98} & 0.96 & 0.96 & 0.95\\

Satimage & \textbf{0.98} & 0.74 & 0.95 & 0.86\\

Segmentation & \textbf{0.54} & 0.39 & 0.51 & 0.49\\

Wave Noise & \textbf{0.51} & 0.50 & 0.50 & 0.48\\

WBC & \textbf{0.50} & 0.47 & 0.48 & 0.47\\

WBCD & \textbf{0.99} & 0.92 & \textbf{0.99} & 0.91\\

\midrule
		
Average & \textbf{0.75} & 0.66 & 0.73 & 0.69\\

\bottomrule
\end{tabular}
\caption{[Higher is better] Outlier detection results (AUC scores) on real-world data sets.} \label{tab:outlier}
\end{table}